\documentclass{article}

\usepackage[utf8]{inputenc} %
\usepackage[T1]{fontenc}    %
\usepackage{hyperref}       %
\usepackage{url}            %
\usepackage{booktabs}       %
\usepackage{amsfonts}       %
\usepackage{nicefrac}       %
\usepackage{microtype}      %
\usepackage{xcolor}         %
\usepackage{fullpage}
\usepackage{natbib}

\usepackage{amsmath}
\usepackage{amsthm}
\usepackage{minitoc}
\usepackage{enumitem}
\usepackage{graphicx}

\usepackage{selectp}

\input{imports/macros.sty}

\title{Pareto Frontiers in Neural Feature Learning: \\ Data, Compute, Width, and Luck}

\author{Benjamin L. Edelman$^1$ \quad Surbhi Goel$^{2}$ \quad Sham Kakade$^{1}$ \quad Eran Malach$^3$\quad Cyril Zhang$^4$\\
\vspace{-2mm} \\
  \normalsize{$^1$Harvard University\quad $^2$University of Pennsylvania } \\ \normalsize{$^3$Hebrew University of Jerusalem \quad $^4$Microsoft Research NYC}\\
   \vspace{-2mm} \\
\normalsize{ \texttt{bedelman@g.harvard.edu, surbhig@cis.upenn.edu, sham@seas.harvard.edu,}}\\\normalsize{\texttt{ eran.malach@mail.huji.ac.il, cyrilzhang@microsoft.com} } }
\date{}

\begin{document}

\setlist[itemize]{leftmargin=2em}
\setlist[enumerate]{leftmargin=2em}

\doparttoc %
\faketableofcontents %

\maketitle

\begin{abstract}
In modern deep learning, algorithmic choices (such as width, depth, and learning rate) are known to modulate nuanced resource tradeoffs. This work investigates how these complexities necessarily arise for feature learning in the presence of computational-statistical gaps. We begin by considering offline sparse parity learning, a supervised classification problem which admits a statistical query lower bound for gradient-based training of a multilayer perceptron. This lower bound can be interpreted as a \emph{multi-resource tradeoff frontier}: 
successful learning can only occur if one is sufficiently rich (large model), knowledgeable (large dataset), patient (many training iterations), or lucky (many random guesses).
We show, theoretically and experimentally, that sparse initialization and increasing network width yield significant improvements in sample efficiency in this setting. Here, width plays the role of parallel search: it amplifies the probability of finding ``lottery ticket'' neurons, which learn sparse features more sample-efficiently.
Finally, we show that the synthetic sparse parity task can be useful as a proxy for real problems requiring axis-aligned feature learning. We demonstrate improved sample efficiency on tabular classification benchmarks by using wide, sparsely-initialized MLP models; these networks sometimes outperform tuned random forests.
\end{abstract}

\section{Introduction}
\label{sec:intro}

Algorithm design in deep learning can appear to be more like ``hacking'' than an engineering practice. Numerous architectural choices and training heuristics can affect various performance criteria and resource costs in unpredictable ways. Moreover, it is understood that these multifarious hyperparameters all interact with each other; as a result, the task of finding the ``best'' deep learning algorithm for a particular scenario is foremost empirically-driven. When this delicate balance of considerations is achieved (i.e. when deep learning works well), learning is enabled by phenomena which cannot be explained by statistics or optimization in isolation. It is natural to ask: \emph{is this heterogeneity of methods and mechanisms necessary?}

\begin{figure}
    \centering
    \includegraphics[width=\linewidth]{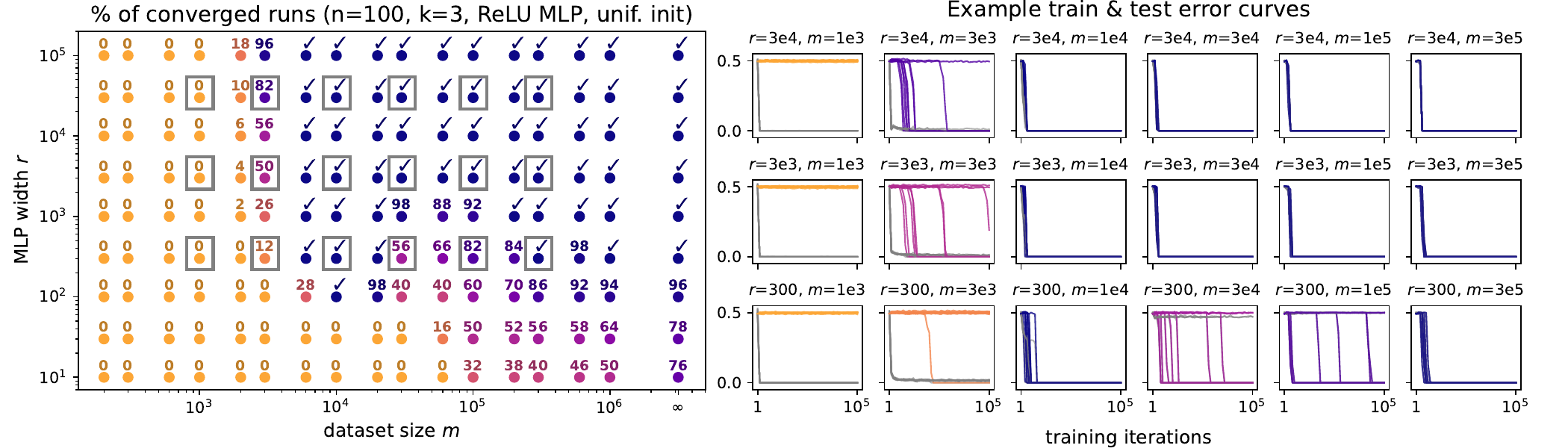}
    \caption{\textbf{``More data or larger network?''} Effects of jointly scaling the dataset and model sizes, for $2$-layer MLPs trained to learn sparse parities. \emph{Left:} A success frontier, where the computational challenge of feature learning can be surmounted by increasing sample size \emph{or} model size. Each point shows the percentage of successful training runs (out of 50 trials) for each (dataset size $m$, width $r$); $(\checkmark)$ denotes 100\% successful trials. \emph{Right:} 10 example training curves (\textcolor{gray}{gray = training error}; \textcolor{violet}{colored = test error}) for each of the boxed $(m,r)$ pairs. See main text and Section~\ref{sec:experiments} for further details.
    }
    \label{fig:main-overview}
\end{figure}

This work studies a single synthetic binary classification task in which the above complications are recognizable, and, in fact, \emph{provable}. This is the problem of offline (i.e. small-sample) sparse parity learning: identify a $k$-way multiplicative interaction between $n$ Boolean variables, given $m$ random examples. We begin by interpreting the standard statistical query lower bound for this problem as a \emph{multi-resource tradeoff frontier} for deep learning, balancing between the heterogeneous resources of dataset size, network size, number of iterations, and success probability. We show that in different regimes of simultaneous resource constraints (data, parallel computation, sequential computation, and random trials), the standard algorithmic choices in deep learning can succeed by diverse and entangled mechanisms.
Specifically, our contributions are as follows:

\paragraph{Multi-resource lower and upper bounds.} We formulate a ``data $\times$ width $\times$ time $\times$ luck'' lower bound for offline sparse parity learning with feedforward neural nets (Theorem~\ref{thm:lower_bound}). This barrier arises from the classic statistical query (SQ) lower bound for this problem. We show that under different resource constraints, the tractability of learning can be ``bought'' with varied mixtures of these resources. In particular, in Theorems~\ref{thm:oversparse_upper_bound} and \ref{thm:undersparse} we prove that by tuning the width and initialization scheme, we can populate this frontier with a spectrum of successful models ranging from narrow networks requiring many training samples to sample-efficient networks requiring many neurons, as summarized by the following informal theorem statement:

\begin{inftheorem}
Consider the problem of learning $(n,k)$-parities from $m$ i.i.d. samples with a $2$-layer width-$r$ ReLU MLP, whose first-layer neurons are initialized with sparsity $s$. After $O_n(1)$ steps of gradient descent, the $k$ relevant coordinates are identified with probability 0.99 when \emph{(1)} $s > \Omega(k)$, $r = \Theta((n/s)^k)$ and $m = \Theta(n^2(s/k)^{k-1})$, and when \emph{(2)} $s< k$, $r = \Theta((n/k)^s)$ and $m = \Theta((n/k)^{k-s-1})$.
\end{inftheorem}

Intuitively, this analysis reveals a feature learning mechanism by which overparameterization (i.e. large network width) plays a role of parallel search over randomized subnetworks. Each individual hidden-layer neuron has its own sample complexity for identifying the relevant coordinates, based on its \emph{Fourier gap} \citep{barak2022hidden} at initialization. Trained with parallel gradient updates, the full network implicitly acts as an ensemble model over these neurons, whose overall sample complexity is determined by the ``winning lottery tickets'' \citep{frankle2018lottery} (i.e. the lucky neurons initialized to have the lowest sample complexity). This departs significantly from the \emph{neural tangent kernel} \citep{jacot2018neural} regime of function approximation with wide networks, in which overparameterization \emph{removes} data-dependent feature selection (rather than parallelizing it across subnetworks).

\paragraph{Empirical study of neural nets' statistical thresholds for sparse parity learning.}
We corroborate the theoretical analysis with a systematic empirical study of offline sparse parity learning using SGD on MLPs, demonstrating some of the (perhaps) counterintuitive effects of width, data, and initialization. For example, Figure~\ref{fig:main-overview} highlights our empirical investigation into the interactions between data, width, and success probability.  The left figure shows the fractions of successful training runs as a function of dataset size (x-axis) and width (y-axis). Roughly, we see a ``success frontier'', where having a larger width can be traded off with smaller sample sizes. The right figure depicts some training curves (for various widths and sample sizes). Grokking \citep{power2022grokking,liu2022omnigrok} (discontinuous and delayed generalization behavior induced by optimization dynamics) is evident in some of these figures.

\paragraph{``Parity2real'' transfer of algorithmic improvements.} 
It is often observed that deep learning methods underperform tree-based methods (e.g. random forests) on tabular datasets, particularly those where the target function depends on a few of the input features in a potentially non-smooth manner; see \cite{grinsztajn2022why} for a recent discussion. Motivated by our findings in the synthetic parity setting (and the observation that the sparse parity task possesses a number of the same properties of these problematic real-world datasets), we then turn to experimentally determine the extent to which our findings also hold for real tabular data. We evaluate MLPs of various depths and initialization sparsities on 16 tabular classification tasks, which were standardized by \cite{grinsztajn2022why} to compare neural vs. tree-based methods. Figure~\ref{fig:tabular-main} shows that wider networks and sparse initialization yield improved performance, as in the parity setting. In some cases, our MLPs outperform tuned random forests.

\subsection{Related work}

In the nascent empirical science of large-scale deep learning, \emph{scaling laws}
\citep{kaplan2020scaling,hoffmann2022training} have been shown to extrapolate model performance with remarkable consistency, revealing flexible tradeoffs and Pareto frontiers between the heterogeneous resources of data and computation. The present work reveals that in the simple synthetic setting of parity learning, the same intricacies can be studied theoretically and experimentally. In particular, viewing \emph{model size} $\times$ \emph{training iterations} $\times$ \emph{random restarts} as a single \emph{``total FLOPs''} resource, our study explains why data $\times$ compute can be a necessary and sufficient resource, through the lens of SQ complexity.

\paragraph{Analyses of deep feature learning.} Formally characterizing the representation learning mechanisms of neural networks is a core research program of deep learning theory. Many recent works have analyzed gradient-based feature learning \citep{wei2019regularization,barak2022hidden,zhenmei2022theoretical,abbe2022merged,damian2022neural, telgarsky2022feature}, escaping the ``lazy'' neural tangent kernel (NTK) regime \citep{jacot2018neural, chizat2019lazy}, in which features are fixed at initialization.

\paragraph{Learning parities with neural networks.}
The XOR function has been studied as an elementary challenging example since the dawn of artificial neural networks \citep{minsky1969introduction}, and has been revisited at various times: e.g. neural cryptography \citep{rosen2002mutual}; learning interactions via hints in the input distribution \citep{daniely2020learning}; a hard case for self-attention architectures \citep{hahn2020theoretical}.
Closest to this work, \citet{barak2022hidden} find that in the case of \emph{online} (infinite-data) parity learning, SGD provides a feature learning signal for a \emph{single neuron}, and can thus converge at a near-optimal computational rate for non-overparameterized networks. They note computational-statistical tradeoffs and grokking in the offline setting, which we address systematically. \citet{merrill2023tale} examine the same problem setting empirically, investigating a mechanism of competing sparse and dense sub-networks. \citet{abbe2023sgd} provide evidence that the time complexity required by an MLP to learn an arbitrary sparse Boolean function is governed by the largest ``leap'' in the \textit{staircase} of its monomials, each of which is a sparse parity. \cite{telgarsky2022feature} gives a margin-based analysis of gradient flow on a two-layer neural network that achieves improved sample complexity ($\tilde{O}(n/\epsilon)$) for the 2-sparse parity problem, at the cost of exponential width.

\paragraph{Neural nets and axis-aligned tabular data. } Decision tree ensembles such as random forests \citep{breiman2001random} and XGBoost \citep{chen2016xgboost} remain more popular among practitioners than neural networks on tabular data \citep{kaggle2021state}, despite many recent attempts to design specialized deep learning methods \citep{borisov2022deep}. Some of these employ sparse networks \citep{yang2022locally, lutz2022sparse} similar to those considered in our theory and experiments.

We refer the reader to Appendix~\ref{sec:app-related} for additional related work.
\section{Background}
\label{sec:background}

\subsection{Parities, and parity learning algorithms}
\label{subsec:parity-background}

A light bulb is controlled by $k$ out of $n$ binary switches; each of the $k$ influential switches can toggle the light bulb's state from any configuration. The task in parity learning is to identify the subset of $k$ important switches, given access to $m$ i.i.d. uniform samples of the state of all $n$ switches. Formally, for any $1 \leq k \leq n$ and $S \subseteq [n]$ such that $|S| = k$, the parity function $\chi_S : \{\pm1\}^n \rightarrow \pm 1$ is defined as $\chi_S(x_{1:n}) := \prod_{i\in S} x_i$.\footnote{Equivalently, parity can be represented as a function of a bit string $b \in \{0, 1\}^n$, which computes the XOR of the influential subset of indices $S$: $\chi_S(b_{1:n}) := \oplus_{i\in S} b_i$.} The $(n,k)$-parity learning problem is to identify $S$ from samples $(x \sim \mathrm{Unif}(\{\pm1\}^n, y = \chi_S(x)) )$; i.e. output a classifier with $100$\% accuracy on this distribution, without prior knowledge of $S$.

This problem has a rich history in theoretical computer science, information theory, and cryptography. There are several pertinent ways to think about the fundamental role of parity learning:
\begin{enumerate}%
    \item[(i)] \textbf{Monomial basis elements:} A $k$-sparse parity is a degree-$k$ monomial, and thus the analogue of a Fourier coefficient with ``frequency'' $k$ in the harmonic analysis of Boolean functions \citep{o2014analysis}. Parities form an important basis for polynomial learning algorithms (e.g. \citet{andoni2014learning}).
    \item[(ii)] \textbf{Computational hardness:} There is a widely conjectured \emph{computational-statistical gap} for this problem \citep{applebaum2009fast,applebaum2010public}, which has been proven in restricted models such as SQ \citep{kearns1998efficient} and streaming \citep{kol2017time} algorithms. The statistical limit is $\Theta( \log( \text{number of possibilities for } S )) = \Theta(k \log n)$ samples, but the amount of computation needed (for an algorithm that can tolerate an $O(1)$ fraction of noise) is believed to scale as $\Omega(n^{c k})$ for some constant $c > 0$ -- i.e., there is no significant improvement over trying all subsets.
    \item[(iii)] \textbf{Feature learning:} This setting captures the learning of a concept which depends jointly on multiple attributes of the inputs, where the lower-order interactions (i.e. correlations with degree-$k' < k$ parities) give no information about $S$. Intuitively, samples from the sparse parity distribution look like random noise until the learner has identified the sparse interaction. This notion of feature learning complexity is also captured by the more general \emph{information exponent} \citep{arous2021online}.
\end{enumerate}

\subsection{Notation for neural networks}

For single-output regression, a 2-layer multi-layer perceptron (MLP) is a function class, parameterized by a matrix $W \in \mathbb{R}^{r \times n}$, vectors $v,b \in \mathbb{R}^{ r}$ and scalar $\beta \in \reals$, defined by
$$\widehat y : x \mapsto v^\top \sigma(Wx+b)+\beta,$$
where $\sigma$ is an activation function, usually a scalar function applied identically to each input. For rows $w_i$ of $W$, the intermediate variables $\sigma(w_i^\top x+b_i)$ are thought of as \emph{hidden layer} activations or \emph{neurons}. The number of parallel neurons is often called the \emph{width}.
\section{Theory}
\label{sec:theory}

In this section we theoretically study the interactions between data, width, time, and luck on the sparse parity problem. We begin by rehashing Statistical Query (SQ) lower bounds for parity learning in the context of gradient-based optimization, showing that without sufficient resources sparse parities cannot be learned. Then, we prove upper bounds showing that parity learning is possible with correctly scaling either width (keeping sample size small), sample size (keeping width small), or a mixture of the two.

\paragraph{Statistical query algorithms.}
A seminal work by \cite{kearns1998efficient} introduced the statistical query (SQ) algorithm framework, which provides a means of analyzing noise-tolerant algorithms. Unlike traditional learning algorithms, SQ algorithms lack access to individual examples but can instead make queries to an \emph{SQ oracle}, which responds with noisy estimates of the queries over the population. Notably, many common learning algorithms, including noisy variants of gradient descent, can be implemented within this framework.
While SQ learning offers robust guarantees for learning in the presence of noise, there exist certain problems that are learnable from examples but not efficiently learnable from statistical queries \citep{blum1994weakly,blum2003noise}. One notable example of such a problem is the parity learning problem, which possesses an SQ lower bound. This lower bound can be leveraged to demonstrate the computational hardness of learning parities with gradient-based algorithms (e.g., \cite{shalev2017failures}).

\subsection{Lower bound: a multi-resource hardness frontier} \label{subsec:pareto}
 We show a version of the SQ lower bound in this section. Assume we optimize some model $h_\theta$,
where $\theta \in \reals^r$ (i.e., $r$ parameters).
Let $\ell$ be some loss function satisfying: $\ell'(\hat{y}, y) = -y + \ell_0(\hat{y})$;
for example, one can choose $\ell_2(\hat{y}, y) = \frac{1}{2} (y-\hat{y})^2$.
Fix some hypothesis $h$, target $f$, a sample $\cs \subseteq \{\pm 1\}^n$ and distribution $\cd$ over $\{\pm 1\}^n$. We denote the empirical loss by $L_\cs(h,f) = \frac{1}{\abs{\cs}} \sum_{ \x \in \cs} \ell(h(\x), f(\x))$ and the population loss by $L_\cd(h,f) := \E_{\x \sim \cd} \left[\ell(h(\x),f(\x))\right]$.
We consider SGD updates of the form:
\[
\theta_{t+1} = \theta_{t} - \eta_t \left(\nabla_\theta \left(L_{\cs_t}(h_{\theta_t},f) + R(\theta_t)\right) + \xi_t\right),
\]
for some sample $\cs_t$, step size $\eta_t$,
regularizer $R(\cdot)$,
and adversarial noise $\xi_t \in [-\tau, \tau]^r$. For normalization, we assume that $\norm{\nabla h_\theta(\x)}_\infty \le 1$ for all $\theta$ and $\x \in \cx$.

Denote by $\mathbf{0}$ the constant function, mapping all inputs to $0$. Let $\theta^\star_t$ be the following trajectory of SGD that is independent
of the target. Define $\theta^\star_t$ recursively
s.t. $\theta^\star_0 = \theta_0$ and
\[
    \theta^\star_{t+1} = \theta^\star_t -
    \eta_t \nabla_\theta \left( L_\cd(h_{\theta^\star_t},\mathbf{0}) + R(\theta^\star_t)\right)
\]

\begin{assumption}[Bounded error for gradient estimator]
    \label{asm:grad_est}
    For all $t$, suppose that
    \[\norm{\nabla L_{\cs_t}(h_{\theta^\star_t}, \mathbf{0}) - \nabla L_{\cd}(h_{\theta^\star_t}, \mathbf{0})}_\infty \le \tau/2.\]
\end{assumption}

\textit{Remark:}
If $m = \tilde{O}(1/\tau^2)$,\footnote{We use the notation $\tilde{O}$ to hide logarithmic factors.}
Assumption \ref{asm:grad_est} is satisfied w.h.p. for: 1)  $\cs_t \sim \cd^m$  (``online'' SGD),
2) $\cs_t = \cs$ for some $\cs \sim \cd^m$ (``offline'' GD)
and 3) $\cs_t$ is a batch of size $m$
 sampled uniformly at random from $\cs$,
 where $\cs \sim \cd^M$
 and $M \ge m$ (``offline'' SGD).
 Indeed, in all these cases we have $\E_{\cs_t}\left[\nabla L_{\cs_t}(h_{\theta^\star_t}, \mathbf{0})\right] = \nabla L_\cd(h_{\theta^\star_t},\mathbf{0})$, and the above follows from standard concentration bounds.

The lower bound in this section uses standard statistical query arguments to show that gradient-based algorithms, \emph{without sufficient resources}, will fail to learn $(n,k)$-parities. We therefore start by stating the four types of resources that impact learnability:
\begin{itemize}
\setlength\itemsep{0em}
    \item The \emph{number of parameters} $r$ -- equivalently, the number of parallel ``queries''.
    \item The \emph{number of gradient updates} $T$ -- i.e. the serial running time of the training algorithm.
    \item The \emph{gradient precision} $\tau$ -- i.e. how close the empirical gradient is to the population gradient. As discussed above, $\tilde O(1/\tau^2)$ samples suffice to obtain such an estimator.
    \item The probability of success $\delta$.
\end{itemize}
The following theorem ties these four resources together, showing that without a sufficient allocation of these resources, gradient descent will fail to learn:

\begin{proposition}
    \label{thm:lower_bound}
    Assume that $\theta_0$ is randomly drawn from some distribution. For every $r,T,\delta,\tau > 0$, if $\frac{rT}{\tau^2 \delta} \le \frac{1}{2}\binom{n}{k}$, then there exists some $(n,k)$-parity s.t. with probability at least $1-\delta$ over the choice of $\theta_0$, the first $T$ iterates of SGD are statistically independent of the target function.
\end{proposition}

The proof follows standard SQ lower bound arguments \citep{kearns1998efficient, feldman2008evolvability}, and for completeness is given in Appendix \ref{app-sec:proofs}. The core idea of the proof is the observation that any gradient step has very little correlation (roughly $n^{-k}$) with most sparse-parity functions, and so there exists a parity function that has small correlation with all steps. In this case, the noise can force the gradient iterates to follow the trajectory $\theta^\star_1, \dots, \theta^\star_T$, which is independent of the true target. We note that, while the focus of this paper is on sparse parities, similar analysis applies for a broader class of functions that are characterized by large SQ dimension (see \cite{blum1994weakly}).

Observe that there are various ways for an algorithm to escape the lower bound of Theorem \ref{thm:lower_bound}. We can scale a single resource with $\binom{n}{k}$, keeping the rest small, e.g. by training a network of size $n^{k}$, or using a sample size of size $n^{k}$. Crucially, we note the possibility of \emph{interpolating} between these extremes: one can spread this homogenized ``cost'' across multiple resources, by (e.g.) training a network of size $n^{k/2}$ using $n^{k/2}$ samples. In the next section, we show how neural networks can be tailored to solve the task in these interpolated resource scaling regimes.

\subsection{Upper bounds: many ways to trade off the terms}

As a warmup, we discuss some simple SQ algorithms that succeed in learning parities by properly scaling the different resources. First, consider deterministic exhaustive search, which computes the error of all possible $(n,k)$-parities, choosing the one with smallest error. This can be done with constant $\tau$ (thus, sample complexity logarithmic in $n$), but takes $T=\Theta(n^k)$ time. Since querying different parities can be done in parallel, we can reduce the number of steps $T$ by increasing the number of parallel queries $r$. Alternatively, it is possible to query only a randomly selected subset of parities, which reduces the overall number of queries but increases the probability of failure.

The above algorithms give a rough understanding of the frontier of algorithms that succeed at learning parities. However, at first glance, they do not seem to reflect algorithms used in deep learning, and are specialized to the parity problem. In this section, we will explore the ability of neural networks to achieve similar tradeoffs between the different resources. In particular, we focus on the interaction between sample complexity and network size, establishing learning guarantees with \emph{interpolatable mixtures} of these resources.

Before introducing our main positive theoretical results, we discuss some prior theoretical results on learning with neural networks, and their limitations in the context of learning parities. Positive results on learning with neural networks can generally be classified into two categories: those that reduce the problem to convex learning of linear predictors over a predefined set of features (e.g. the NTK), and those that involve neural networks departing from the kernel regime by modifying the fixed features of the initialization, known as the feature learning regime.

\paragraph{Kernel regime.} When neural networks trained with gradient descent stay close to their initial weights, optimization behaves like kernel regression on the neural tangent kernel \citep{jacot2018neural,du2018gradient}:
the resulting function is approximately of the form $\x \mapsto \inner{\psi(\x), \bw}$, where $\psi$ is a \emph{data-independent} infinite-dimensional embedding of the input, and $\bw$ is some weighting of the features. However, it has been shown that the NTK (more generally, any fixed kernel) cannot achieve low $\ell_2$ error on the $(n,k)$-parity problem, unless the sample complexity grows as $\Omega(n^k)$ (see \cite{kamath2020approximate}). Thus, no matter how we scale the network size or training time, neural networks trained in the NTK regime cannot learn parities with low sample complexity, and thus do not enjoy the flexibility of resource allocation discussed above.

\paragraph{Feature learning regime.} 
Due to the limitation of neural networks trained in the kernel regime, some works study learning in the ``rich'' regime, quantifying how hidden-layer features adapt to the data. Among these, \cite{barak2022hidden} analyze a feature learning mechanism requiring exceptionally small network width: SGD on 2-layer MLPs can solve the \emph{online} sparse parity learning problem with network width \emph{independent} of $n$ (dependent only on $k$), at the expense of requiring a suboptimal ($\approx n^{k}$) number of examples. This mechanism is \emph{Fourier gap amplification}, by which SGD through a \emph{single neuron} $\sigma(w^\top x)$ can perform feature selection in this setting, via exploiting a small gap between the relevant and irrelevant coordinates in the population gradient.
The proof of Theorem \ref{thm:oversparse_upper_bound} below relies on a similar analysis, extended to the offline regime (i.e., multiple passes over a dataset of limited size).

\subsubsection{``Data $\!\times\!$ model size'' success frontier for sparsely-initialized MLPs}
\label{subsubsec:sparse-init-theory}

In this section, we analyze a 2-layer MLP with ReLU ($\sigma(x) = \max(0,x)$) activation, trained with batch (``offline'') gradient-descent over a sample $\cs$ with $\ell_2$-regularized updates\footnote{For simplicity, we do not assume adversarial noise in the gradients as in the lower bound. However, similar results can be shown under bounded noise.}:
\[
\theta^{(t+1)} = (1-\lambda^{(t)})\theta^{(t)} - \eta^{(t)} \nabla L_\cs(h_{\theta^{(t)}})
\]
We allow learning rates $\eta$, and weight decay coefficients $\lambda$ to differ between layers and iterations. For simplicity, we analyze the case where no additional noise is added to each update; however, we believe that similar results can be obtained in the noisy case (e.g., using the techniques in \cite{feldman2017statistical}). Finally, we focus on ReLU networks with $s$-sparse initialization of the first layer: every weight $\bw_i$ has $s$ randomly chosen coordinates set to $1$, and the rest set to $0$. Note that after initialization, all of the network's weights are allowed to move, so sparsity is not necessarily preserved during training.

\paragraph{Over-sparse initialization ($s>\Omega(k)$).} The following theorem demonstrates a ``data $\!\times\!$ width'' success frontier when learning $(n,k)$-parities with sparsely initialized ReLU MLPs at sparsity levels $s > \Omega(k)$.
\begin{theorem}
    \label{thm:oversparse_upper_bound}
    Let $k$ be an even integer, and $\epsilon \in (0,1/2)$. Assume that $n \geq \Omega(1/\epsilon^2)$. For constants $c_1, c_2, c_3, c_4$ depending only on $k$, choose the following: (1) \textit{sparsity level}: $s \ge c_1/\epsilon^2$, for some odd $s$, (2) width of the network: $r =  c_2(n/s)^k$, (3) sample size: $m \ge c_3 (s/k)^{k-1}n^2 \log n$, and (4) number of iterations: $T \ge c_4/\epsilon^2$.
    Then, for every $(n,k)$-parity distribution $\cd$, with probability at least $0.99$ over the random samples and initialization, gradient descent with these parameter settings returns a function $h_T$ s.t.
    $L_\cd(h_T) \le \epsilon$.
\end{theorem}

Intuitively, by varying the sparsity parameter in Theorem \ref{thm:oversparse_upper_bound}, we obtain a family of algorithms which smoothly interpolate between the \emph{small-data/large-width} and \emph{large-data/small-width} regimes of tractability. First, consider a sparsity level linear in $n$ (i.e. $s = \alpha\cdot n$). In this case, a small network (with width $r$ independent of the input dimension) is sufficient for solving the problem, but the sample size must be large ($\Omega(n^{k+1})$) for successful learning; this recovers the result of \cite{barak2022hidden}. At the other extreme, if the sparsity is independent of $n$, the sample complexity grows only as $O(n^2 \log n)$\footnote{We note that the additional $n^2$ factor in the sample complexity can be removed if we apply gradient truncation, thus allowing only a logarithmic dependence on $n$ in the small-sample case.}, but the requisite width becomes $\Omega(n^k)$.

\textit{Proof sketch.} The proof of Theorem \ref{thm:oversparse_upper_bound} relies on establishing a Fourier anti-concentration condition, separating the relevant (i.e. indices in $S$) and irrelevant weights in the initial population gradient, similarly as the main result in \cite{barak2022hidden}. 
When we initialize an $s$-sparse neuron, there is a probability of $\gtrsim (s/n)^k$ that the subset of activated weights contains the ``correct'' subset $S$. In this case, to detect the subset $S$ via the Fourier gap, it is sufficient to observe $s^{k-1}$ examples instead of $n^{k-1}$. Initializing the neurons more sparsely makes it less probable to draw a \emph{lucky} neuron, but once we draw a lucky neuron, it requires fewer samples to find the right features. Thus, increasing the width reduces overall sample complexity, by sampling a large number of ``lottery tickets''.

\paragraph{Under-sparse initialization ($s < k$).} The sparsity parameter can modulate similar data vs. width tradeoffs for feature learning in the ``under-sparse'' regime. We provide a partial analysis for this more challenging case, showing that one step of gradient descent can recover a correct subnetwork. Appendix~\ref{subsec:under-feature-selection-proof} discusses the mathematical obstructions towards obtaining an end-to-end guarantee of global convergence.
\begin{theorem}
\label{thm:undersparse}
    For even $k, s$, sparsity level $s < k$, network width $r = O((n/k)^s)$, and $\eps$-perturbed\footnote{For ease of analysis, we use a close variant of the sparse initialization scheme: $s$ coordinates out of the $n$ coordinates are chosen randomly and set to 1, and the rest of the coordinates are set to $\epsilon < (n - s)^{-1}$. Without the small norm dense component in the initialization, the population gradient will be 0 at initialization.} $s$-sparse random initialization scheme s.t. for every $(n,k)$-parity distribution $\cd$, with probability at least $0.99$ over the choice of sample and initialization after one step of batch gradient descent (with gradient clipping) with sample size $m = O((n/k)^{k-s-1})$ and appropriate learning rate, there is a subnetwork in the ReLU MLP that approximately computes the parity function.
\end{theorem}
Here, the sample complexity can be improved by a factor of $n^s$, at the cost of requiring the width to be $n^s$ times larger. The proof of Theorem \ref{thm:undersparse} relies on a novel analysis of improved Fourier gaps with ``partial progress'': intuitively, if a neuron is randomly initialized with a subset of the relevant indices $S$, it only needs to identify $k-s$ more coordinates, inheriting the improved sample complexity for the $(n - s, k - s)$-parity problem. Note that the probability of finding such a lucky neuron scales as $(n/k)^{-s}$, which governs how wide (number of lottery tickets) the network needs to be.

\paragraph{Remarks on the exact-sparsity regime.}
Our theoretical analyses do not extend straightforwardly to the case of $s = \Theta(k)$. Observe that if the value of $k$ is known, initializing a network of size $r = \widetilde{O}\left(n^k\right)$  with sparsity $s=k$ gives w.h.p. a subnetwork with good first-layer features at initialization. We believe that with a proper choice of regularization and training scheme, it is possible to show that such a network learns to select this subnetwork with low sample complexity. We leave the exact details of this construction and end-to-end proofs for future work.

\paragraph{Analogous results for dense initialization schemes?}
We believe that the principle of ``parallel search with randomized per-subnetwork sample complexities'' extends to other initialization schemes (including those more commonly used in practice), and leads to analogous success frontiers. To support this, our experiments investigate both sparse and uniform initialization, with qualitatively similar findings. For dense initializations, the mathematical challenge lies in analyzing the Fourier anti-concentration of general halfspaces (see the discussion and experiments in Appendix C.1 of \citep{barak2022hidden}). The axis-aligned inductive biases imparted by sparse initialization may also be of independent practical interest.
\section{Experiments}
\label{sec:experiments}

A high-level takeaway from Section~\ref{sec:theory} is that when a learning problem is computationally difficult but statistically easy, a complex frontier of resource tradeoffs can emerge; moreover, it is possible to interpolate between extremes along this frontier using ubiquitous algorithmic choices in deep learning, such as overparameterization, random initialization, and weight decay. In this section, we explore the nature of the frontier with an empirical lens---first with end-to-end sparse parity learning, then with natural tabular datasets.

\subsection{Empirical Pareto frontiers for offline sparse parity learning}
\label{subsec:synthetic-experiments}

We launch a large-scale ($\sim$200K GPU training runs) exploration of resource tradeoffs when training neural networks to solve the offline sparse parity problem. While Section~\ref{sec:theory} analyzes idealized variants of SGD on MLPs which interpolate along the problem's resource tradeoff frontier, in this section we ask whether the same can be observed end-to-end with standard training and regularization.

\begin{figure}
    \centering
    \raisebox{3.3em}{\includegraphics[width=0.18\linewidth]{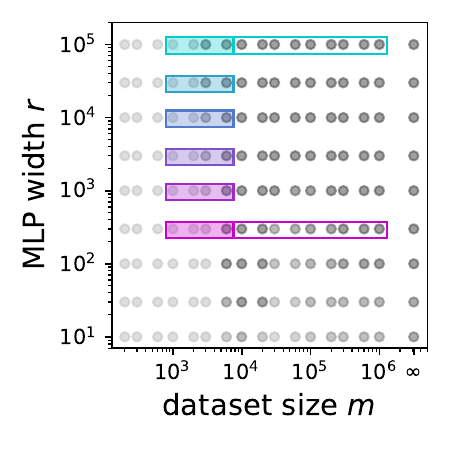}}
    \hspace{-0.2em}
    \raisebox{0.2em}{\includegraphics[width=0.42\linewidth]{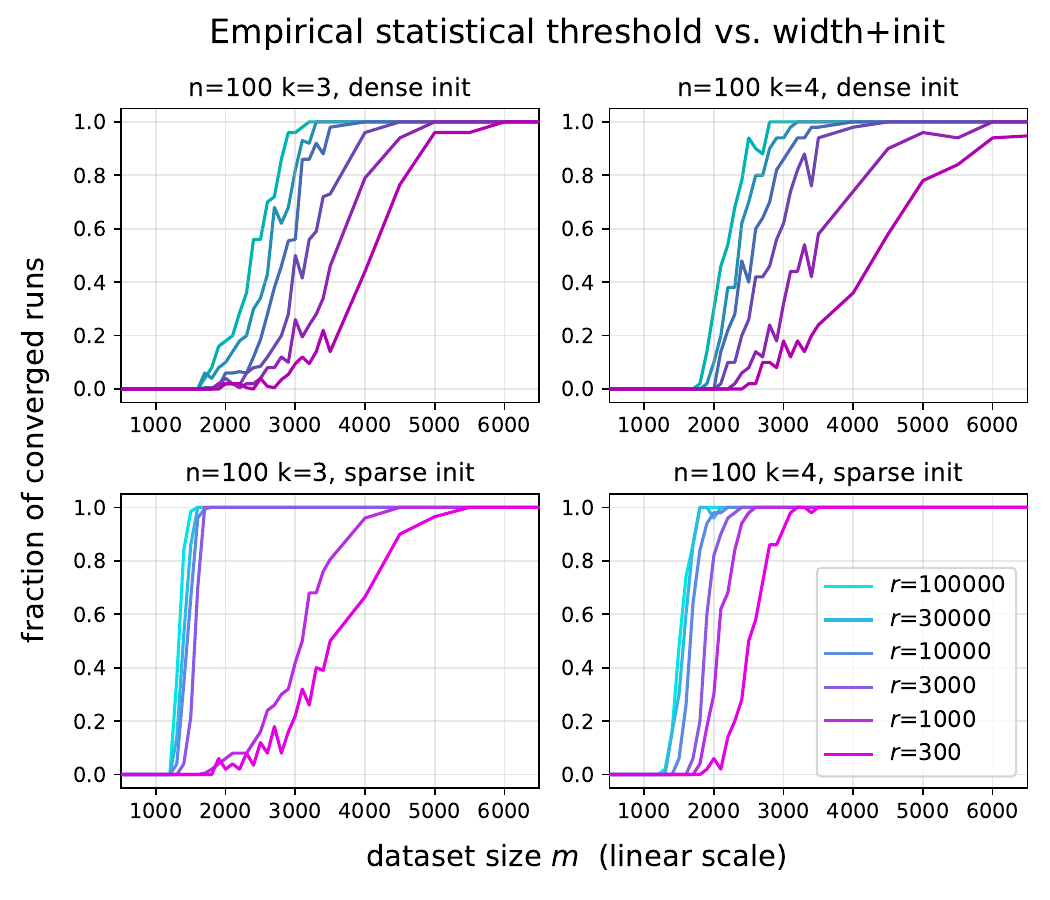}}
    \includegraphics[width=0.38\linewidth]
    {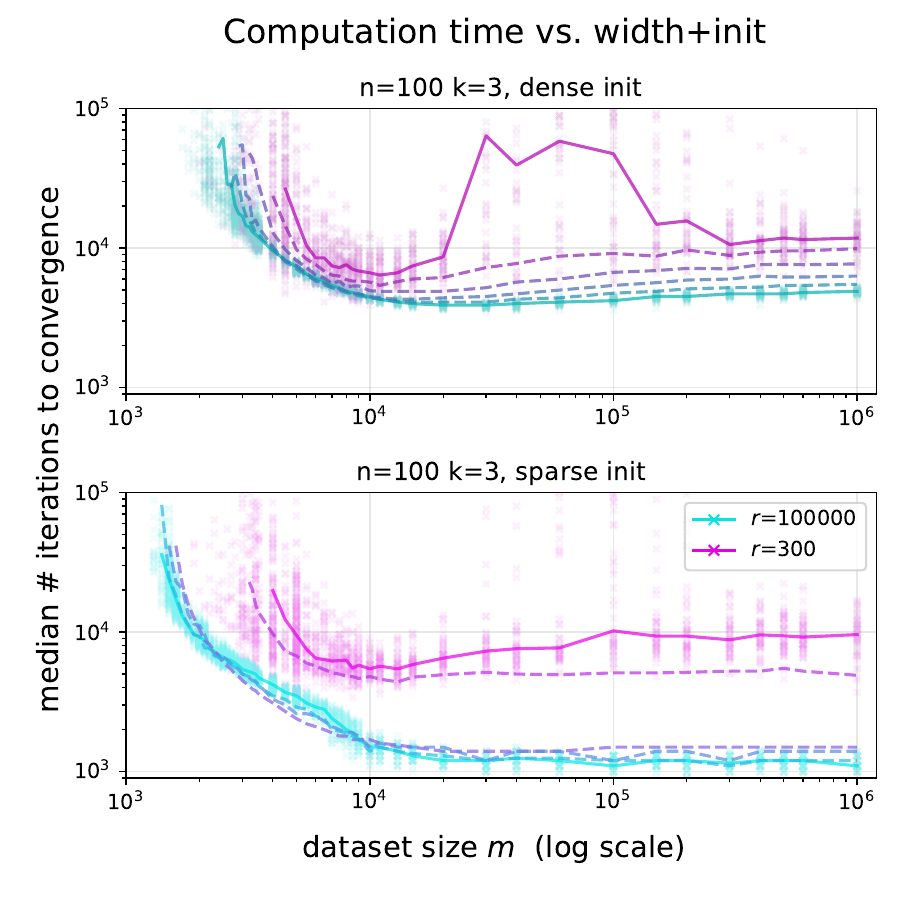}
    \caption{Zoomed-in views of the interactions between width, data, time, and luck. \emph{Left:} Locations of these runs in the larger parameter space. \emph{Center:} Success probability vs. dataset size. In accordance with our theory, \textbf{width buys luck, and improves end-to-end sample efficiency}, despite increasing the network's capacity. \emph{Right:} Number of iterations to convergence vs. dataset size. We observe a \textbf{data vs. time tradeoff} in the grokking regime (at the edge of feasibility), as well as a \textbf{``sample-wise double descent''} performance drop with more data (which can also be seen in Figure~\ref{fig:main-overview}, and disappears with larger widths). Comparing dense vs. sparse initializations (upper vs. lower plots; sparse inits are colored more brightly), we see computational and statistical benefits of the sparse initialization scheme from Section~\ref{subsubsec:sparse-init-theory}. }
    \label{fig:main-detail}
\end{figure}

On various instances of the $(n,k)$-sparse parity learning problem, we train a 2-layer MLP with identical hyperparameters, varying the network width $r \in \{10, 30, 100, \ldots, 10^5\}$ and the dataset size $m \in \{100, 200, 300, 600, 1000, \ldots, 10^6\}$. Alongside standard algorithmic choices, we consider one non-standard 
augmentation of SGD: the \emph{under-sparse} initialization scheme from Section \ref{subsubsec:sparse-init-theory}; we have proven that these give rise to ``lottery ticket'' neurons which learn the influential coordinates more sample-efficiently.
Figure~\ref{fig:main-overview} (in the introduction) and Figure~\ref{fig:main-detail} illustrate our findings at a high level;
details and additional discussion are in Appendix~\ref{subsec:data-width-experiments}). We list our key findings below:

\begin{enumerate}
    \item[(1)] \textbf{A ``success frontier'': large width can compensate for small datasets.} We observe convergence and perfect generalization when $m \ll n^k$. In such regimes, which are far outside the online setting considered by \citet{barak2022hidden}, high-probability sample-efficient learning is enabled by large width. This can be seen in Figure~\ref{fig:main-overview} (left), and analogous plots in Appendix~\ref{subsec:data-width-experiments}.
    \item[(2)] \textbf{Width is monotonically beneficial, and buys data, time, and luck.} In this setting, increasing the model size yields exclusively positive effects on success probability, sample efficiency, and the number of serial steps to convergence (see Figure~\ref{fig:main-detail}). This is a striking example where end-to-end generalization behavior runs \emph{opposite} to uniform convergence-based upper bounds, which predict that enlarging the model's capacity \emph{worsens} generalization.
    \item[(3)] \textbf{Sparse axis-aligned initialization buys data, time, and luck.} Used in conjunction with a wide network, we observe that a sparse, axis-aligned initialization scheme yields strong improvements on all of these axes; see Figure~\ref{fig:main-detail} (bottom row). In smaller hyperparameter sweeps, we find that $s = 2$ (i.e. initialize every hidden-layer neuron with a random $2$-hot weight vector) works best.
    \item[(4)] \textbf{Intriguing effects of dataset size.} As we vary the sample size $m$, we note two interesting phenomena; see Figure~\ref{fig:main-detail} (right). The first is grokking \citep{power2022grokking}, which has been previously documented in this setting \citep{barak2022hidden,merrill2023tale}. This entails a \emph{data vs. time} tradeoff: for small $m$ where learning is marginally feasible, optimization requires significantly more training iterations $T$. Our second observation is a ``sample-wise double descent'' \citep{nakkiran2021deep}: success probability and convergence times can worsen with \emph{increasing} data. Both of these effects are also evident in Figure~\ref{fig:main-overview}).
\end{enumerate}

\paragraph{Lottery ticket neurons.} The above findings are consistent with the viewpoint taken by the theoretical analysis, where randomly-initialized SGD plays the role of \emph{parallel search}, and a large width increases the number of random subnetworks available for this process---in particular, the ``winning lottery ticket'' neurons, for which feature learning occurs more sample-efficiently. To provide further evidence that sparse subnetworks are responsible for learning the parities, we perform a smaller-scale study of network prunability in Appendix~\ref{subsec:lottery}.

\subsection{Sample-efficient deep learning on natural tabular datasets}

Sparse parity learning is a toy problem, in that it is defined by an idealized distribution, averting the ambiguities inherent in reasoning about real-life datasets. However, due to its provable hardness (Theorem~\ref{thm:lower_bound}, as well as the discussion in Section~\ref{subsec:parity-background}), it is a \emph{maximally hard} toy problem in a rigorous sense\footnote{Namely, its SQ dimension is equal to the number of hypotheses, which is what leads to Theorem~\ref{thm:lower_bound}.}. In this section, we perform a preliminary investigation of how the empirical and algorithmic insights gleaned from Section~\ref{subsec:synthetic-experiments} can be transferred to more realistic learning scenarios.

To this end, we use the benchmark assembled by \cite{grinsztajn2022why}, a work which specifically investigates the performance gap between neural networks and tree-based classifiers (e.g. random forests, gradient-boosted trees), and includes a standardized suite of 16 classification benchmarks with numerical input features. The authors identify three common aspects of tabular\footnote{``Tabular data'' refers to the catch-all term for data sources where each coordinate has a distinct semantic meaning which is consistent across points.} tasks which present difficulties for neural networks, especially vanilla MLPs:
\begin{itemize}
    \item[(i)] \textbf{The feature spaces are not rotationally invariant.} In state-of-the-art deep learning, MLPs are often tasked with function representation in rotation-invariant domains (token embedding spaces, convolutional channels, etc.).
    \item[(ii)] \textbf{Many of the features are uninformative.} In order to generalize effectively, especially from a limited amount of data, it is essential to avoid overfitting to these features.
    \item[(iii)] \textbf{There are meaningful high-frequency/non-smooth patterns in the target function.} Combined with property (i), decision tree-based methods (which typically split on axis-aligned features) can appear to have the ideal inductive bias for tabular modalities of data.
\end{itemize}

\begin{figure}
    \centering
    \includegraphics[width=0.98\linewidth]{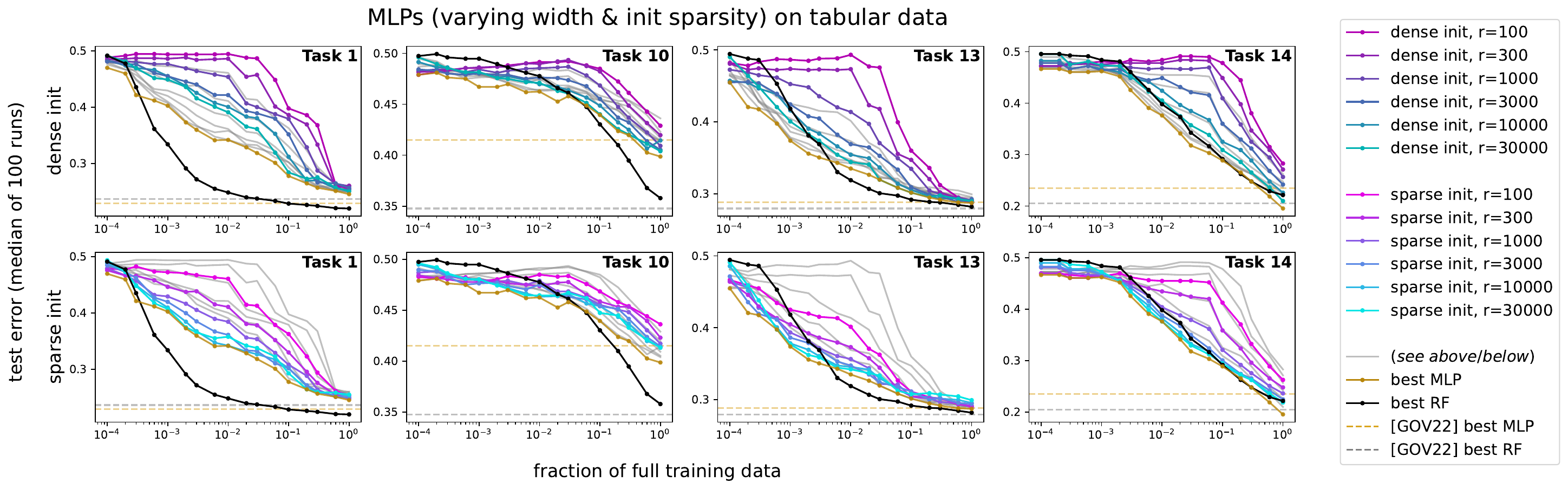}
    \caption{Analogous investigation of MLP width $r$ and sparse initialization for real-world tabular datasets (OpenML benchmarks assembled by \citet{grinsztajn2022why}), varying the dataset size $m$ via downsampling. Large width and sparse initialization tend to improve generalization, in accordance with the theory and experiments for synthetic parity tasks. In some settings, our best MLPs outperform tuned random forests. Dotted lines denote the test errors reported by \cite{grinsztajn2022why} of tuned MLPs and RFs on the full datasets. Full results on all 16 tasks are in Appendix~\ref{subsec:tabular-experiments}.}
    \label{fig:tabular-main}
\end{figure}

Noting that the sparse parity task possesses all three of the above qualities, we conduct a preliminary investigation on whether our empirical findings in the synthetic case carry over to natural tabular data. In order to study the impact of algorithmic choices (mainly width and sparse initialization) on sample efficiency, we create low-data problem instances by subsampling varying fractions of each dataset for training.
Figure~\ref{fig:tabular-main} provides a selection of our results. We note the following empirical findings, which are the tabular data counterparts of results (2) and (3) in Section~\ref{subsec:synthetic-experiments}:
\begin{itemize}[leftmargin=3em]
    \item[(2T)] \textbf{Wide networks generalize on small tabular datasets.} Like in the synthetic experiments, width yields nearly monotonic end-to-end benefits for learning. This suggests that the ``parallel search + pruning'' mechanisms analyzed in our paper are also at play in these settings. In some (but not all) cases, these MLPs perform competitively with tuned tree-based classifiers.
    \item[(3T)] \textbf{Sparse axis-aligned initialization sometimes improves end-to-end performance.} This effect is especially pronounced on datasets which are downsampled to be orders of magnitude smaller. We believe that this class of drop-in replacements for standard initialization merits further investigation, and may contribute to closing the remaining performance gap between deep learning and tree ensembles on small tabular datasets.
\end{itemize}

\section{Conclusion}

We have presented a theoretical and empirical study of offline sparse parity learning with neural networks; this is a provably hard problem which admits a multi-resource lower bound in the SQ model. We have shown that the lower bound can be surmounted using varied mixtures of these resources, which correspond to natural algorithmic choices and scaling axes in deep learning. By investigating how these choices influence the empirical ``success frontier'' for this hard synthetic problem, we have arrived at some promising improvements for MLP models of tabular data (namely, large width and sparse initialization).
These preliminary experiments suggest that a more intensive, exhaustive study of algorithmic improvements for MLPs on tabular data has a chance of reaping significant rewards, perhaps even surpassing the performance of decision tree ensembles.

\paragraph{Broader impacts and limitations.} The nature of this work is foundational; the aim of our theoretical and empirical investigations is to contribute to the fundamental understanding of neural feature learning, and the influences of scaling relevant resources. A key limitation is that our benchmarks on tabular data are only preliminary; it is a significant and perennial methodological challenge to devise fair and comprehensive comparisons between neural networks and tree-based learning paradigms.

\section*{Acknowledgements}

We are grateful to Boaz Barak for helpful discussions and to Matthew Salganik for helpful comments on a draft version. This work has been made possible in part by a gift from the Chan Zuckerberg Initiative Foundation to establish the Kempner Institute for the Study of Natural and Artificial Intelligence.
Sham Kakade acknowledges funding from the Office of Naval Research under award N00014-22-1-2377. Ben Edelman acknowledges funding from the National Science Foundation Graduate Research Fellowship Program under award \#DGE 2140743.

\bibliography{bib}
\bibliographystyle{apalike}

\renewcommand{\theequation}{\thesection.\arabic{equation}}

\newpage
\appendix

\addcontentsline{toc}{section}{Appendix} %
\part{Appendix} %
\parttoc %
\newpage

\section{Additional related work}\label{sec:app-related}
\paragraph{Learning parities with neural networks.} Parities (or XORs) have been shown to be computationally hard to learn for SQ algorithms including gradient-based methods. Various works make additional assumptions to avoid these hardness results to show that neural networks can be efficiently trained to learn parities \citep{daniely2020learning,shi2021theoretical,frei2022random,malach2021quantifying}. More recently, \citet{barak2022hidden} have focused on understanding how neural networks training on parities without any additional assumption behaves at this computational-statistical limit. They show that one step of gradient descent on a single neuron is able to recover the indices corresponding to the parity with $n^{O(k)}$ samples/computation. \citet{abbe2023sgd} improve this bound to $O(n^{k-1})$ online SGD steps and generalize the result to handle hierarchical staircases of parity functions which requires a multi-step analysis. \cite{telgarsky2022feature} studies the problem of 2-sparse parities with two-layer neural networks trained with vanilla SGD (unlike our restricted two-step training algorithm) and studies the margins achieved post training. They use the margins to get optimal sample complexity $\tilde{O}(n^2/\epsilon)$ in the NTK regime. Going beyond NTK, they analyze gradient flow (with certain additional modifications) on an exponential wide 2-layer network (making it computationally inefficient) to get the improved sample complexity of $\tilde{O}(n/\epsilon)$. In contrast to this, our goal is to improve sample complexity while maintaining computational efficiency, using random guessing via the sparse initialization. 

\paragraph{Learning single-index/multi-index models over Gaussians with neural networks.} Another line of work \citep{arous2021online,ba2022high,damian2022neural,bietti2022learning, damian2023smoothing} has focused on learning functions that depend on a few directions, in particular, single-index and multi-index models over Gaussians using neural nets. These can be thought of as a continuous analog to our sparse parity problem. In a similar analysis (as parities) of online SGD for single index models, \cite{arous2021online} propose the notion of an \textit{information exponent} which captures the initial correlation between the model and the target function, and get convergence results similar to the parity setting with sample complexity $O(n^{k-1})$ for information exponent $k$ (can be thought similar to the $k$ in the parity learning problem). \cite{damian2023smoothing} improve this result by showing that a smoothed version of GD achieves the optimal sample complexity (for CSQ algorithms) of $(n^{k/2})$. Going beyond CSQ algorithms, \cite{chen2020learning} provide a filtered-PCA algorithm that achieves polynomial dependence on the dimension $n$ in both compute and sample complexity. Note that this is not achievable for CSQ algorithms. For the parity learning problem, the SQ computational lower bounds are $\Omega(n^k)$ (as described in Section \ref{subsec:pareto}). 

\paragraph{Empirical inductive biases of large MLPs.}
Our experiments on tabular benchmarks suggest that wide and sparsely-initialized vanilla MLPs can sometimes close the performance gap between neural networks and decision tree ensemble methods. This corroborates recent findings that vanilla MLPs have strong enough inductive biases to generalize nontrivially in natural data modalities, despite the overparameterization and lack of architectural biases via convolution or recurrence.
Notably, many state-of-the-art computer vision models have removed convolutions \citep{dosovitskiy2020image,tolstikhin2021mlp}; recently, \citep{bachmann2023scaling} demonstrate that even large vanilla NLPs can compete with convolutional models for image classification. \citet{yang2022tensor} find monotonic improvements in terms of model width, which are stabilized by their theoretically-motivated hyperparameter scaling rules.

\paragraph{Multi-resource scaling laws for deep learning.} Many empirical studies \citep{kaplan2020scaling,henighan2020scaling,hoffmann2022training,zhai2022scaling}, motivated by the pressing need to allocate resources effectively in large-scale deep learning, corroborate the presence and regularity of neural scaling laws. Precise statements and hypotheses vary; \citet{kaplan2020scaling} fit power-law expressions which predict holdout validation log-perplexity of a language model in terms of dataset size, model size, and training iterations ($m, r, T$ in our notation). The present work shows how such a joint dependence on $m \times r \times T$ can arise from a single feature learning problem with a computational-statistical gap. Numerous works attempt to demystify neural scaling laws with theoretical models \citep{bahri2021explaining,hutter2021learning,michaud2023quantization}; ours is unique in that it does not suppose a long-tailed data distribution (the statistical complexity of identifying a sparse parity is benign). We view these accounts to be mutually compatible: we \emph{do not} purport that statistical query complexity is the unique origin of neural scaling laws, nor that there is a \emph{single} such mechanism.
\section{Proofs}
\label{app-sec:proofs}

\subsection{Multi-resource lower bound for sparse parity learning}
\label{subsec:lower-bound-proof}
For some target function $f$ and some parameters $\theta$, we denote the population gradient over the distribution $\cd$ by:
\[
g(f, \theta) = \E_{\x \sim \cd}\left[\nabla_\theta \ell(h_\theta(\x),f(\x))\right]
\]
and we denote by $g_i(\cdot, \cdot)$ the gradient w.r.t. the $i$-th coordinate of $\theta$.

Similarly, denote the empirical gradient by:
\[
\hat{g}(f, \theta) = \frac{1}{m} \sum_{\x \in \cs} \nabla_\theta \ell(h_\theta(\x), f(\x))
\]
and $\hat{g}_i(\cdot, \cdot)$ denotes the $i$-th coordinate of the empirical gradient.

\begin{lemma}
    \label{lem:var_bound}
    For every $\theta$ and every $i$ it holds that
    \[
        \E_{S \sim \binom{n}{k}}
        \left[\left(g_i(\chi_S, \theta)-\ell_0(h_\theta(\x))\cdot\frac{\partial}{\partial \theta_i}h_\theta(\x)\right)^2\right]
        \le \frac{1}{\binom{n}{k}}
    \]
\end{lemma}

\begin{proof}
    Fix some $i \in [r]$,
    \begin{align*}
        &\E_{S \sim \binom{n}{k}}\left[\left(g_i(\chi_S, \theta)-\ell_0(h_\theta(\x))\cdot\frac{\partial}{\partial \theta_i}h_\theta(\x)\right)^2\right] \\
        &= \E_{S \sim \binom{n}{k}}\left[\E_{\x \sim \cd}\left[\chi_S(\x)\cdot\frac{\partial}{\partial \theta_i}h_\theta(\x)\right]^2\right] \\
        &= \frac{1}{\binom{n}{k}} \sum_{S \in \binom{n}{k}} \E_{\x \sim \cd}\left[\chi_S(\x)\cdot\frac{\partial}{\partial \theta_i}h_\theta(\x)\right]^2
        \le \frac{1}{\binom{n}{k}}
    \end{align*}
    where the last inequality is from Parseval,
    using the assumption $\norm{\nabla h_\theta(\x)}_\infty \le 1$.

\end{proof}

\begin{proof}[Proof of Proposition \ref{thm:lower_bound}]
    Using the Lemma \ref{lem:var_bound} we get:
    \begin{align*}
        &\E_{S \sim \binom{n}{k}}
        \left[\max_{i,t}\left(g_i(\chi_S, \theta^*_t)-\ell_0(h_{\theta^*_t}(\x))\cdot\frac{\partial}{\partial \theta_i}h_{\theta^*_t}(\x)\right)^2\right] \\
        &\le \E_{S \sim \binom{n}{k}}
        \left[\sum_{i=1}^r \sum_{t=1}^T \left(g_i(\chi_S, \theta^*_t)-\ell_0(h_{\theta^*_t}(\x))\cdot\frac{\partial}{\partial \theta_i}h_{\theta^*_t}(\x)\right)^2\right] \\
        &\le \frac{rT}{\binom{  n}{k}} \le \delta \tau^2
    \end{align*}
    Therefore, taking expectation over the choice of $\theta_0$
    \begin{align*}
        &\E_{\theta_0} \E_{S \sim \binom{n}{k}} \left[\max_{i,t}{\left(g_i(\chi_S, \theta^*_t)-\ell_0(h_{\theta^*_t}(\x))\cdot\frac{\partial}{\partial \theta_i}h_{\theta^*_t}(\x)\right)}^2\right] \\
        &= \E_{S \sim \binom{n}{k}} \E_{\theta_0} \left[\max_{i,t}{\left(g_i(\chi_S, \theta^*_t)-\ell_0(h_{\theta^*_t}(\x))\cdot\frac{\partial}{\partial \theta_i}h_{\theta^*_t}(\x)\right)}^2\right] \le \delta \tau^2/2
    \end{align*}
    So, there exists some $S \in \binom{n}{k}$ s.t.
    \[
    \E_{\theta_0} \left[\max_{i,t}{\left(g_i(\chi_S, \theta^*_t)-\ell_0(h_{\theta^*_t}(\x))\cdot\frac{\partial}{\partial \theta_i}h_{\theta^*_t}(\x)\right)}^2\right] \le \delta \tau^2/2
    \]
    Observe the noise variable $\xi_t =  \ell_0(h_{\theta^*_t}(\x))\cdot\frac{\partial}{\partial \theta_i}h_{\theta^*_t}(\x)-\hat{g}_i(\chi_S, \theta^*_t,\cs_t)$.
    From Markov's inequality and Assumption \ref{asm:grad_est}, with probability at least $1-\delta$ over the choice of $\theta_0$, for all $t \le T$ and $i \in [r]$:
    \[
        \abs{\hat{g}_i(\chi_S, \theta^*_t,\cs_t)- \ell_0(h_{\theta^*_t}(\x))\cdot\frac{\partial}{\partial \theta_i}h_{\theta^*_t}(\x)} \le \tau 
    \]
    therefore, we get that $\xi_1, \dots, \xi_T \in [-\tau,\tau]^r$ (i.e., this is a valid choice of adversarial noise variables), and SGD follows the trajectory $\theta_1^\star, \dots, \theta^\star_T$.
\end{proof}

\subsection{Feature selection with an over-sparse initialization and a wide network}
\label{subsec:feature-selection-proof}
\subsubsection{Warmup: existence of good subnetworks}
Let $h_\bw(\x) = \sigma(\inner{\bw, \x})$ be a single ReLU neuron, where $\sigma(x) = \max \{x,0\}$. Fix some $4k < s \le n$. Assume we initialize $\bw \in \{0,1\}^n$ by randomly choosing $s$ coordinates and setting them to $1$, and setting the rest to zero. Fix some subset $S \subseteq \binom{n}{k}$. We say that $\bw$ is a \emph{good} neuron if $S \subseteq \bw$. We say that $\bw$ is a \emph{bad} neuron if it is not a \emph{good} neuron.

\begin{lemma}
    \label{lem:good_neuron_prob}
    With probability at least $\left(s/2n\right)^k$ over the choice of $\bw$, $\bw$ is a \emph{good} neuron.
\end{lemma}

\begin{proof}
    There are $\binom{n}{s}$ choices for $\bw$, and there are $\binom{n-k}{s-k}$ \emph{good} choices for $\bw$. Observe that:
    \[
    \binom{n}{s} = \frac{n(n-1)\cdots(n-k+1)}{s(s-1)\cdots(s-k+1)}\binom{n-k}{s-k} \le \left(\frac{2n}{s}\right)^k \binom{n-k}{s-k}
    \]
    and therefore the required follows.
\end{proof}

\begin{lemma}
    \label{lem:bound_good}
    Assume $k$ is even, $s$ is odd and $k < \frac{s}{4}$. There exist constants $C_{k}, c_{k}$ s.t. if $\bw$ is a \emph{good} neuron, then:
    \begin{enumerate}
        \item For all $i \in S$,  $$\E_{\x} \left[ \frac{\partial}{\partial w_i}h_\bw(\x) \cdot \chi_S(\x) \right] = C_{k,s}$$
        \item For all $i \notin S$, 
        $$\E_{\x} \left[ \frac{\partial}{\partial w_i}h_\bw(\x) \cdot \chi_S(\x) \right] = w_i c_{k,s}$$
    \end{enumerate}
    and furthermore, there exists a constant $\kappa_k$ s.t. $\abs{C_{k,s}} > \kappa_k \binom{s}{k-1}^{-1/2}$ and $\frac{\abs{c_{k,s}}}{\abs{C_{k,s}}} \le \frac{4k}{s}$.
\end{lemma}

\begin{proof}
    First, consider the case where $i \in S$. Therefore,
    \begin{align*}
        \E_{\x \sim \{\pm 1\}^n} \left[ \frac{\partial}{\partial w_i}h_\bw(\x) \cdot \chi_S(\x) \right]
        &= \E_{\x \sim \{\pm 1\}^n} \left[\sigma'(\inner{\bw,\x}) \cdot x_i \cdot \chi_S(\x) \right] \\
        &= \E_{\x \sim \{\pm 1\}^n} \left[\ind_{\left\lbrace\sum_{j \in \bw}x_j \ge 0\right\rbrace} \chi_{S \setminus \{i\}}(\x)\right] \\
        &= \E_{\x \sim \{\pm 1\}^{s}}\left[\left(\frac{1}{2}\MAJ_{s}(\x)+\frac{1}{2}\right) \chi_{S \setminus \{i\}}(\x)\right] \\
        &= \frac{1}{2}\widehat{\MAJ_{s}}(S \setminus \{i\})
    \end{align*}
    where $S \setminus \{i\}$ is interpreted as a subset of $[s]$. 
    From symmetry of the Majority function, all Fourier coefficients of the same order are equal. Therefore, the first condition holds for $C_{k,s} = \frac{1}{2} \widehat{\MAJ_s}(k-1)$, where $\widehat{\MAJ_s}(k-1)$ denotes the $(k-1)$-th order Fourier coefficient.

    When $i \in S \setminus \bw$ we get:
    \begin{align*}
        \E_{\x \sim \{\pm 1\}^n} \left[ \frac{\partial}{\partial w_i}h_\bw(\x) \cdot \chi_S(\x) \right]
        &= \E_{\x \sim \{\pm 1\}^n} \left[\sigma'(\inner{\bw,\x}) \cdot x_i \cdot \chi_S(\x) \right] \\
        &= \E_{\x \sim \{\pm 1\}^n} \left[\ind_{\left\lbrace\sum_{j \in \bw}x_j > 0\right\rbrace} \chi_{S \cup \{i\}}(\x)\right] \\
        &= \E_{\x \sim \{\pm 1\}^{s}}\left[\left(\frac{1}{2}\MAJ_{s}(\x)+\frac{1}{2}\right) \chi_{S \cup \{i\}}(\x)\right] \\
        &= \frac{1}{2}\widehat{\MAJ_{s}}(S \cup \{i\})
    \end{align*}
    Finally, when $i \notin S$ we get:
    \begin{align*}
        \E_{\x \sim \{\pm 1\}^n} \left[ \frac{\partial}{\partial w_i}h_\bw(\x) \cdot \chi_S(\x) \right]
        &= \E_{\x \sim \{\pm 1\}^n} \left[\sigma'(\inner{\bw,\x}) \cdot x_i \cdot \chi_S(\x) \right] \\
        &= \E \left[\sigma'(\inner{\bw,\x}) \cdot \chi_S(\x) \right] \E_{x_i}\left[x_i\right] = 0
    \end{align*}
    Therefore, the second condition holds with $c_k = \frac{1}{2} \widehat{\MAJ_s}(k+1)$.
    
    Now, from Theorem 5.22 in \cite{o2014analysis} we have:
    \[
    \binom{s}{k-1} \cdot \widehat{\MAJ_s}(k-1)^2 = \sum_{S' \in \binom{s}{k-1}} \widehat{\MAJ_s}(S')^2 \ge \rho(k-1)
    \]
    where $\rho(v) = \frac{2}{\pi v2^{v}} \binom{v-1}{\frac{v-1}{2}}$. So, we get $\abs{C_{k,s}} \ge \frac{1}{2}\sqrt{\rho(k-1)} \binom{s}{k-1}^{-1/2}$. Using the same Theorem, we also have:
    \[
    \binom{s}{k+1} \cdot \widehat{\MAJ_s}(k+1)^2 = \sum_{S' \in \binom{s}{k+1}} \widehat{\MAJ_s}(S')^2 \le 2\rho(k+1) < 2 \rho(k-1)
    \]
    So, we get:
    \begin{align*}
        \frac{\abs{c_{k,s}}}{\abs{C_{k,s}}} &\le \frac{\sqrt{2\rho(k-1)}\binom{s}{k+1}^{-1/2}}{\sqrt{\rho(k-1)}\binom{s}{k-1}^{-1/2}} = \sqrt{2} \sqrt{\frac{\binom{s}{k+1}}{\binom{s}{k-1}}} \\
        &= \sqrt{\frac{2k(k+1)}{(s-k+2)(s-k+1)}} \le \sqrt{\frac{4k^2}{(1/4)s^2}} = 4\frac{k}{s}
    \end{align*}
\end{proof}

\begin{lemma}
\label{lem:bound_bad}
    Assume that $k < \frac{s}{4}$. If $\bw$ is a \emph{bad} neuron, then:
    \[
    \norm{\E_\x \left[\frac{\partial}{\partial \bw} h_\bw(\x) \cdot \chi_S(\x)\right]}_1 \le C_{k,s}
    \]
\end{lemma}

\begin{proof}
    First, assume that $|S \setminus \bw| \ge 2$. In this case, there exist $i,i' \in S$ s.t. $w_i = w_{i'}= 0$ and $i \ne i'$. Fix some $j \in [n]$, and choose some $j' \in {i,i'}$ s.t. $j' \ne j$.
        \begin{align*}
        \E_{\x \sim \{\pm 1\}^n} \left[ \frac{\partial}{\partial w_j}h_\bw(\x) \cdot \chi_S(\x) \right]
        &= \E_{\x \sim \{\pm 1\}^n} \left[\sigma'(\inner{\bw,\x})  \cdot \chi_S(\x) \cdot x_j\right] \\
        &= \E_{\x \sim \{\pm 1\}^n} \left[\sigma'(\inner{\bw, \x})  \cdot \chi_{S \setminus \{j'\}}(\x) \cdot x_j x_{j'}\right] \\
        &= \E_{x_j'} \left[x_{j'}\right] \cdot \E_{\x_{[n] \setminus \{j'\}}} \left[\sigma'(\inner{\bw, \x})  \cdot \chi_{S \setminus \{j'\}}(\x) \cdot x_j\right] = 0
    \end{align*}
    and this gives the required.

    Now, assume that $S \setminus \bw = \{i\}$ for some index $i$. For every $j \ne i$, similarly to the previous analysis, we have:
    \[
    \E_{\x \sim \{\pm 1\}^n} \left[ \frac{\partial}{\partial w_j}h_\bw(\x) \cdot \chi_S(\x) \right] = 0
    \]
    Finally, similarly to the proof of Lemma \ref{lem:bound_good}, we have:
    \begin{align*}
        \E_{\x \sim \{\pm 1\}^n} \left[ \frac{\partial}{\partial w_i}h_\bw(\x) \cdot \chi_S(\x) \right]
        &= \frac{1}{2}\widehat{\MAJ_{s}}(S \setminus \{i\}) = C_{k,s}
    \end{align*}
    and so we get the required.
\end{proof}

\subsubsection{End-to-end result}
We train the following network:
\[
f_{\W, \bb, \bu, \beta}(\x) = \sum_{i=1}^r u_i \sigma\left(\inner{\bw_i, \x} + b_i\right) + \beta
\]
We fix some $k \le s < n$ and initialize the network as follows:
\begin{itemize}
    \item Randomly initialize $\bw_1, \dots, \bw_{r/2} \in \{0,1\}^n$ s.t. $\norm{\bw_i}_1 = s$ (i.e., each $\bw_i$ has $s$ active coordinates), with a uniform distribution over all $\binom{n}{s}$ subsets.
    \item Randomly initialize $b_1, \dots, b_{r/2} \sim \left\lbrace\beta_1, \dots, \beta_{k/2-1} \right\rbrace$ where $\beta_i = \frac{1}{2k} \left(-k+2i+1/16\right)$.
    \item Randomly initialize $u_1, \dots, u_{r/2} \sim \{\pm 1\}$ uniformly at random.
    \item Initialize $\bw_{r/2+1}, \dots, \bw_r, b_{r/2+1}, \dots, b_r$ s.t. $\bw_i = \bw_{i-r/2}$ and $b_i = b_{i-r/2}$ (symmetric initialization).
    \item Initialize $u_{r/2+1}, \dots, u_{r}$ s.t. $u_i = -u_{i-r/2}$.
    \item Initialize $\beta = 0$
\end{itemize}

Let $\ell(\hat{y},y) = \max (1-y\hat{y},0)$ be the hinge-loss function. Given some distribution $\cd$ over $\cx \times \{\pm 1\}$, define the loss of $f$ over the distribution by:
\[
L_\cd(f) = \E_{(\x,y)\sim \cd}\left[\ell(f(\x), y)\right]
\]
Similarly, given a sample $\cs \subseteq \cx \times \{\pm 1\}$, define the loss of $f$ on the sample by:
\[
L_\cs(f) = \frac{1}{\abs{\cs}} \sum_{(\x,y) \in \cs} \ell(f(\x),y)
\]
We train the network by gradient descent on a sample $\cs$ with $\ell_2$ regularization (weight decay):
\[
\theta^{(t+1)} = (1-\lambda^{(t)})\theta^{(t)} - \eta^{(t)} \nabla L_\cs(f_{\theta^{(t)}})
\]

We allow choosing learning rate $\eta$, the weight decay $\lambda$ differently for each layer, separately for the weights and biases, and for each iteration.

\begin{lemma}
    \label{lem:grad_concentration}
    Fix some $\tau > 0, \delta > 0$. Let $\cs$ be a set of $m$ examples chosen i.i.d. from $\cd$. Then, if $m \geq \frac{4\log(4nr/\delta)}{\tau^2}$, with probability at least $1-\delta$ over the choice of $\cs$, it holds that:
    \[
    \norm{\nabla_{\W,\bb} L_\cd(f_{\W,\bb,\bu,\beta})-\nabla_{\W,\bb} L_\cs(f_{\W,\bb,\bu,\beta})}_\infty \le \tau
    \]
\end{lemma}

\begin{proof}
Denote by $\theta \in \reals^{nr+r}$ the set of all parameters in $\W,\bb$.
For every $i \in [nr+r]$, from Hoeffding's inequality, we have:
\[
\Pr\left(\abs{\frac{\partial}{\partial \theta_i}L_\cd(f_\theta)-\frac{\partial}{\partial \theta_i}L_\cs(f_\theta)} \ge \tau\right) \le 2 \exp(-m\tau^2/4) \le \frac{\delta}{2nr}
\]
and the required follows from the union bound.
\end{proof}

Given some initialization of $\W,\bb$, for every $j$ denote by $I_j \subseteq [r/2]$ the set of indices of neurons with \emph{good} weights and bias equal to $\beta_j$.
We say that an initialization is $r'$-\emph{good} if for all $j$ we have $r'/2 \le \abs{I_j} \le 2r'$.

Let $g$ be some vector-valued function. We define:
\[
\norm{g}_{\infty,2} = \sup_{x}\norm{g(x)}_2
\]
For some mapping $\psi: \cx \to \reals^r$ and some $B > 0$, denote by $\ch_{\psi,B}$ the class of linear functions of norm at most $B$ over the mapping $\psi$:
\[
\ch_{\psi,B} = \{h_{\psi,\bu} ~:~ \norm{\bu}_2 \le B\}
\]
where $h_{\psi,\bu}(\x) = \inner{\psi(\x), \bu}$.

Denote by $\phi^{(t)}$ the output of the first layer of $f_{\W,\bb,\bu,\beta}$ after $t$ iterations of GD\footnote{We assume $1$ is appended to the vector for allowing bias}.

\begin{lemma}
    \label{lem:first_step}
    Fix some $\delta > 0$. Fix some $r'$-\emph{good} initialization $\W,\bb$. Let $\eta = \frac{1}{2k}\abs{C_{k,s}}^{-1}$, and let $\tau  \le \frac{1}{\eta 16kn}$. Assume that $m \ge \frac{4\log(4nr/\delta)}{\tau^2}$. Then, there exists some mapping $\psi$ s.t. the following holds:
    \begin{enumerate}
        \item $\norm{\psi}_{\infty,2} \le \sqrt{8kr'}$
        \item There exists $\bu^\star$ s.t. $L_\cd(h_{\psi,\bu^\star}) \le \frac{\sqrt{4kr'}B_k}{\sqrt{s}}$, with  $\norm{\bu^\star}_2 \le B_k$ for some constant $B_k$.
        \item With probability at least $1-\delta$ over the choice of $\cs \sim \cd^m$, we have 
        $$\norm{\psi-\phi^{(1)}}_{\infty,2} \le 4kr'n \eta \tau$$
    \end{enumerate}
\end{lemma}

\begin{proof}
    We construct two mappings $\psi,\psi^\star$ as follows.
    \begin{itemize}
        \item We will denote $\psi_0(\x) = \psi^\star_0(\x) = 1$ to allow a bias term.
        \item For every $i$, if $\bw_i$ is a bad neuron, we set $\psi_i(\x) = \psi_i^\star(\x) = 0$.
        \item For every good neuron $\bw_i$ s.t. $i \in I_j$, we set:
        \begin{itemize}
            \item $\psi^\star_i(\x) = \sigma\left(\eta C_{k,s}\sum_{j' \in S}x_{j'} + \beta_j\right)$
            \item $\psi_i(\x) = \sigma\left(\eta C_{k,s}\sum_{j' \in S}x_{j'} +\eta c_{k,s} \sum_{j' \in \bw_i \setminus S} x_{j'}+ \beta_j\right)$
            \item $\psi^\star_{i+r/2}(\x) = \sigma\left(-\eta C_{k,s} x_{j'} + \beta_j\right)$
            \item $\psi_{i+r/2}(\x) = \sigma\left(-\eta C_{k,s}\sum_{j' \in S}x_{j'} +\eta c_{k,s} \sum_{j' \in \bw_i \setminus S} x_{j'}+ \beta_j\right)$
        \end{itemize}
    \end{itemize}
    
    We will show that $\psi^\star$ achieves loss zero, and that $\psi$ approximates it. We assume $C_{k,s} > 0$ and the case of $C_{k,s} < 0$ is derived similarly.

    First, notice that from Lemma \ref{lem:bound_good},
    \[
    \abs{\eta c_{k,s}} = \frac{1}{2k}\frac{\abs{c_{k,s}}}{\abs{C_{k,s}}} \le \frac{2}{s}
    \]

    \textbf{Claim:} for every $\bu$, $\abs{L_\cd(h_{\psi,\bu}) - L_\cd(h_{\psi^\star,\bu})} \le \frac{\norm{\bu}\sqrt{4kr'}}{\sqrt{s}}$

    \textbf{Proof:} Observe that
    \begin{align*}
        \abs{L_\cd(h_{\psi,\bu}) - L_\cd(h_{\psi^\star,\bu})} &\le \E_\cd\left[\abs{\ell(h_{\psi,\bu}(\x),y)-\ell(h_{\psi^\star,\bu}(\x),y)}\right] 
        \le \E_\cd \left[\abs{h_{\psi^\star,\bu}(\x) - h_{\psi,\bu}(\x)}\right] \\
        &= \E_\cd \left[\abs{\inner{\psi(\x)-\psi^\star(\x),\bu}}\right]
        \stackrel{\mathrm{C.S}}{\le} \norm{\bu} \E_\cd \left[\norm{\psi(\x)^\star-\psi(\x)}\right] \\
        &\stackrel{\mathrm{Jensen}}{\le} \norm{\bu} \sqrt{\E_\cd\left[\norm{\psi(\x)^\star-\psi(\x)}^2\right]} 
        = \norm{\bu}\sqrt{\sum_{i}\E_\cd\left[(\psi_i^\star(\x)-\psi_i(\x))^2\right]} \\
        &\le \norm{\bu}\sqrt{\sum_{i~\mathrm{is~good}} \E_\cd \left[\left(\eta c_{k,s}\sum_{j' \in \bw_i\setminus S}x_{j'}\right)^2\right]}
        \le \norm{\bu} \sqrt{4kr'} \abs{\eta c_{k,s}} \sqrt{s}\\
        &\le \frac{\norm{\bu}\sqrt{4kr'}}{\sqrt{s}}
    \end{align*}

    \textbf{Claim:} There exists $\bu^\star$ with norm $\norm{\bu^\star} \le B_k$ and $L_\cd(h_{\psi^\star,\bu^\star}) = 0$.
    
    \textbf{Proof:} For every $\x$, denote $s_\x = \sum_{j\in S}x_j$,
        and observe that $s_\x \in \{-k, -k+2,\dots, k-2,k\} =: \cs$ and $\chi_S(\x) = s_\x \mod 2$.
        For every $j$, denote $v_{j}^+(s) = \sigma(\frac{1}{2k}s+\beta_j)$ and $v_{j}^-(s) = \sigma(-\frac{1}{2k}s+b_j)$.
        Then, for every $s \in \cs$ denote $\bv(s) = (1,v_1^+(s), \dots, v_{k/2-1}^+,v_{1}^-(s), \dots, v_{k/2-1}^-(s))\in \reals^k$.
        Observe that $V = \{\bv(s)\}_{s\in \cs} \in \reals^{k \times k}$ has linearly independent rows,
        and therefore there exists $\nu = (\nu_0, \nu_{1}^+, \dots, \nu_{k/2-1}^+, \nu_{1}^-, \dots, \nu_{k/2-1}^-)\in \reals^k$ s.t. $\bv(s)^\top \nu = s \mod 2$.
        Now, define $\bu^\star$ s.t. for every bad $i$ we set $u^\star_i = 0$,
        and for every $i \in I_j$ we set $u^\star_i = \frac{1}{\abs{I_j}}\nu_{j}^+$ and
        $u^\star_{i+r/2,-} = \frac{1}{\abs{I_j}}\nu_{j}^-$, and set $u^\star_0 = \nu_0$.
        Observe that for every $\x \in \cx$ we have:
        \begin{align*}
            \inner{\psi(\x), \bu^\star} = \nu_0 + \sum_{j} \frac{1}{\abs{I_j}} \sum_{i \in I_j} (v_j^+(s_\x)\nu_j^+ + v_j^-(s_\x)\nu_j^-)
            = \inner{\bv(s_\x)^\top \nu} = s_\x \mod 2 = \chi_S(\x)
        \end{align*}
        and therefore $L_\cd(h_{\psi,\bu^\star}) = 0$. Additionally,
        observe that
        \[
            \norm{\bu^\star}^2 = \nu_0^2+ \sum_j \frac{1}{\abs{I_j}^2}((\nu_j^+)^2 + (\nu_j^-)^2)  \le \norm{\nu}_2^2
        \]
    Now we prove the statements in the main lemma:
    \begin{enumerate}
    \item For every $\x \in \cx$ we have
        \begin{align*}
            \norm{\psi(\x)}^2_2 = \sum_{j} \sum_{i \in I_j} (\psi_i(\x)^2+\psi_{i+r/2}(\x)^2) \le 4 \sum_{j} \abs{I_j} \le 8kr'
        \end{align*}
    \item Follows from the two previous claims.
    \item Assume we choose $\lambda = 1$ for the weights of the first layer, $\lambda=0$ for the biases of the first layer,
    $\eta = \frac{1}{2k}\abs{C_{k,s}}^{-1}$ for the weights of the first layer, and $\eta = 0$ for all other parameters.
    From Lemma \ref{lem:grad_concentration}, w.p. at least $1-\delta$ we have:
    \[
    \norm{\nabla_{\W,\bb} L_\cd(f_{\W,\bb,\bu,\beta})-\nabla_{\W,\bb} L_\cs(f_{\W,\bb,\bu,\beta})}_\infty \le \tau
    \]
    Denote by $\bw^{(1)}_i$ the $i$-th weight after the first gradient step,
    and denote $\bw^\star_i := -\eta\nabla_{\bw_i}L_\cd(f_{\W,\bb,\bu,\beta})$
    and $\widehat{\bw}_i := -\eta\nabla_{\bw_i}L_\cs(f_{\W,\bb,\bu,\beta})$.
    By the choice of $\lambda$, we get $\bw_i^{(1)} = \widehat{\bw}_i$.
    Observe that for all $i$:
    \[
        \norm{\widehat{\bw}_i-\bw^\star_i}_\infty
        = \eta \norm{\nabla_{\bw_i}L_\cs(f_{\W,\bb,\bu,\beta})-\nabla_{\bw_i}L_\cd(f_{\W,\bb,\bu,\beta})} \le \eta \tau
    \]
    \textbf{Claim}: For all $i,j$, if $w^\star_{i,j} = 0$, then $\abs{w^\star_{i,j}-w^{(1)}_{i,j}} \le  \frac{1}{16kn}$.
    
    \textbf{Proof}: We have $\abs{\widehat{w}_{i,j}} \le \eta \tau$, and the claim follows from the fact that $\eta \tau \le \frac{1}{16kn}$.

    First, consider the case where $\bw_i$ is a \emph{bad} neuron.
    In this case, by Lemma \ref{lem:bound_bad}, we have $\norm{\bw^\star_i}_1 \le \frac{1}{2k}$ and $\norm{\bw_i^\star}_0 \le 1$, and from the previous claim we get $\norm{\bw^{(1)}_i-\bw^\star_i}_1 \le \eta \tau + \frac{1}{16k}$.
    Therefore, for all $\x \in \cx$ we get:
    \begin{align*}
        \abs{\inner{\bw^{(1)}_i,\x}} = \abs{\inner{\bw_i^{(1)}-\bw^\star_i+\bw^\star_i, \x}}\le \left(\norm{\bw^\star_i}_1 + \norm{\bw_i^\star-\bw_i^{(1)}}_1\right) \norm{\x}_\infty \le \frac{10}{16k} + \eta\tau
    \end{align*}

    Since at initialization we have $b_i \le -\frac{15}{16k}$, and we have $\eta \tau \le \frac{1}{16k}$, for all $\x \in \cx$:
    \begin{align*}
        \sigma\left(\inner{\bw_i^{(1)},\x}+b_i\right) = 0 = \psi_i(\x)
    \end{align*}

    Now, assume that $\bw_i$ is a \emph{good} neuron.
    In this case, by Lemma \ref{lem:bound_good}, we have
    $$\bw^\star_{i,j} = \begin{cases} \eta C_{k,s} & j \in S \\ w_{i,j} \eta c_{k,s} & j \notin S\end{cases}$$
    Observe that $\psi_i(\x) = \sigma\left(\inner{\bw_i^\star,\x}+b_i\right)$, and therefore:
    \[
    \abs{\psi_i(\x) - \sigma\left(\inner{\bw^{(1)}_i,\x}+b_i\right)} \le \abs{\inner{\bw_i^{(1)}-\bw_i^\star,\x}} \le \norm{\bw_i^{(1)}-\bw_i^\star}_1 \le n\eta \tau
    \]
    
    Similarly, in this case we will get $\abs{\psi_{i+r/2}(\x) - \sigma\left(\inner{\bw^{(1)}_{i+r/2}, \x} + b_{i+r/2}\right)}\le n \eta \tau$.
    Now the required follows from all we showed.
    \end{enumerate}

\end{proof}
 
\begin{lemma}
\label{lem:bound_loss_diff}
    Fix some mappings $\psi, \psi'$ and some $\bw$. Then, for for every distribution $\cd$:
    \[
    \abs{L_\cd(h_{\psi,\bw})-L_\cd(h_{\psi', \bw})} \le \norm{\bw}\norm{\psi-\psi'}_{\infty,2}
    \]
    and for every sample $\cs$:
    \[
    \abs{L_\cs(h_{\psi,\bw})-L_\cs(h_{\psi', \bw})} \le \norm{\bw}\norm{\psi-\psi'}_{\infty,2}
    \]
\end{lemma}

\begin{proof}
    Observe that, since $\ell$ is $1$-Lipschitz:
    \begin{align*}
        |L_{\cd}(h_{\psi',\bw})-L_{\cd}(h_{\psi,\bw})| &\le \E_{(\x,y) \sim \cd}\left[\abs{\ell(h_{\psi',\bw}(\x),y)-\ell(h_{\psi',\bw}(\x),y)}\right] \\
        &\le \E_{\cd} \left[\abs{h_{\psi',\bw}(\x)-h_{\psi,\bw}(\x)}\right] \le \E_{\cd} \left[\abs{\inner{\psi'(\x)-\psi(\x),\bw}}\right] \\
        &\le \E_{\cd}\left[\norm{\psi'(\x)-\psi(\x)} \norm{\bw}\right] \le \norm{\bw}\norm{\psi-\psi'}_{\infty,2}
    \end{align*}
    and similarly we get: $$\abs{L_\cs(h_{\psi',\bw})-L_\cs(h_{\psi,\bw})}\le \norm{\bw} \norm{\psi-\psi'}_{\infty,2}$$    
\end{proof}

\begin{lemma}
\label{lem:generalization_bound}
    Fix some mapping $\psi$, and let $\cs$ be a sample of size $m$ sampled i.i.d. from $\cd$. Then, with probability at least $1-\delta$ over the choice of $\cs$, for every $\psi'$ and for every $h \in \ch_{\psi', B}$, we have:
    \[
    L_\cd(h) \le L_\cs(h) + \frac{(2B\norm{\psi}_{\infty,2}+1)\sqrt{2\log(2/\delta)}}{\sqrt{m}} + 2B \norm{\psi-\psi'}_{\infty,2}
    \]
\end{lemma}

\begin{proof}
    First, observe that using Theorem 26.12 in \cite{shalev2014understanding}, with probability at least $1-\delta$ over the choice of $\cs$, for every $h_{\psi,\bw} \in \ch_{\psi,B}$ we have:
    \[
    L_\cd(h_{\psi,\bw}) \le L_\cs(h_{\psi,\bw}) + \frac{(2B\norm{\psi}_{\infty,2}+1)\sqrt{2\log(2/\delta)}}{\sqrt{m}}
    \]
    In this case, using the previous lemma, for every $\psi'$ and every $h_{\psi', \bw} \in \ch_{\psi', B}$ we have:
    \begin{align*}
        L_\cd(h_{\psi',\bw})
        &\le L_\cd(h_{\psi,\bw}) + B\norm{\psi-\psi'}_{\infty,2} \\ 
        &\le L_\cs(h_{\psi,\bw}) + \frac{(2B\norm{\psi}_{\infty,2}+1)\sqrt{2\log(2/\delta)}}{\sqrt{m}} + B\norm{\psi-\psi'}_{\infty,2} \\
        &\le L_\cs(h_{\psi',\bw}) + \frac{(2B\norm{\psi}_{\infty,2}+1)\sqrt{2\log(2/\delta)}}{\sqrt{m}} + 2B \norm{\psi-\psi'}_{\infty,2}
    \end{align*}
\end{proof}

\begin{lemma}
    \label{lem:full_one_step}
    Fix $\epsilon, \delta \in (0,1/2)$.
    Let $\psi$ be some mapping s.t. there exists $\bw^\star$ satisfying $\norm{\bw^{\star}}\le B$ and $L_\cd(h_{\psi,\bw^\star}) \le \epsilon$. Let $\cs$ be a sample of size $m$ from $\cd$. With probability at least $1-2\delta$ over the choice of $\cs$, there exists a choice of learning rate, weight decay and truncation parameters s.t. if $\norm{\phi^{(1)}-\psi}_{\infty,2} \le \frac{\epsilon}{B}$ and $\frac{T}{\log(T)} \ge \frac{100}{\epsilon^2}\left(\norm{\psi}_{\infty,2}+B^{-1}\right)^2$ and $m \ge \frac{(4\sqrt{2}B\norm{\psi}_{\infty,2}+1)^2\log(2/\delta)}{\epsilon^2}$, GD returns a function $h$ s.t.
    $L_\cd(h) \le 7\epsilon$.
\end{lemma}

\begin{proof}
    Consider the following convex function:
    \[
    L(\bw) = L_\cs(h_{\phi^{(1)}, \bw})+\frac{\lambda}{2} \norm{\bw}^2
    \]
    \textbf{Claim 1}: for all $\x$ and $y$, we have $\abs{\ell(h_{\psi,\bw^\star}(\x),y)} \le B \norm{\psi}_{\infty,2}$.

    \textbf{Proof}: Observe that,
    \[
    \abs{\ell(h_{\psi,\bw^\star}(\x),y)} \le \abs{h_{\psi,\bw^\star}(\x)} =
    \abs{\inner{\psi(\x),\bw^\star}} \le \norm{\psi(\x)}_{2}\norm{\bw^\star} \le B\norm{\psi}_{\infty,2} 
    \]
    
    \textbf{Claim 2}: W.p. at least $1-\delta$ we have $L_\cs(h_{\psi,\bw^\star}) \le \epsilon + \frac{B\norm{\psi}_{\infty,2}\sqrt{2\log(2/\delta)}}{\sqrt{m}}$
    
    \textbf{Proof}: from Hoeffding's inequality, using the previous claim: 
    \[
    \Pr\left[L_S(h) \ge L_\cd(h)+t\right] \le \exp\left(-\frac{mt^2}{2\norm{\psi}_{\infty,2}^2B^2}\right)
    \]
    And therefore,
    \[
    \Pr\left[L_\cs(h_{\psi,\bw^\star})\ge L_\cd(h_{\psi,\bw^\star})+\frac{\norm{\psi}_{\infty,2}B\sqrt{2\log(2/\delta)}}{\sqrt{m}}\right] \le \delta/2
    \]
    and the required follows from the assumption $L_\cd(h_{\psi,\bw^\star}) \le \epsilon$
    
    \textbf{Claim 3}: $L(\bw^\star) \le 2\epsilon + \frac{B\norm{\psi}_{\infty,2}\sqrt{2\log(2/\delta)}}{\sqrt{m}}+ \frac{\lambda B^2}{2}$.
    
    \textbf{Proof}: from the previous claim, we have $L_\cs(h_{\psi,\bw^\star}) \le \epsilon + \frac{B\norm{\psi}_{\infty,2}\sqrt{2\log(2/\delta)}}{\sqrt{m}}$. Using Lemma \ref{lem:bound_loss_diff}, we get $L_\cs(h_{\phi^{(1)},\bw^\star}) \le \epsilon + \frac{B\norm{\psi}_{\infty,2}\sqrt{2\log(2/\delta)}}{\sqrt{m}}+B \norm{\psi-\phi^{(1)}}_{\infty,2}$ and the required follows.

    \textbf{Claim 3}: there exists a step-size schedule for GD s.t. $L_\cs(\bw_T) \le \inf_{\bw} L_\cs(\bw) + \frac{100 (\norm{\psi}_{\infty,2}+\epsilon/B)^2(1+\log(T))}{\lambda T}$.
    
    \textbf{Proof}: Using \cite{shamir2013stochastic} 

    Combining the previous claims, we get:
    \[
    L(\bw_T) \le 2\epsilon + \frac{B\norm{\psi}_{\infty,2}\sqrt{2\log(2/\delta)}}{\sqrt{m}}+\frac{\lambda B^2}{2} +\frac{100 (\norm{\psi}_{\infty,2}+\epsilon/B)^2(1+\log(T))}{\lambda T}
    \]

    Now, choosing $\lambda = \frac{\epsilon}{B^2}$ we get:
    \[
        L(\bw_T) \le 2\epsilon + \frac{B\norm{\psi}_{\infty,2}\sqrt{2\log(2/\delta)}}{\sqrt{m}}+\frac{100 B^2(\norm{\psi}_{\infty,2}+\epsilon/B)^2(1+\log(T))}{\epsilon T}
    \]

    So, if $\frac{T}{\log(T)} \ge \frac{100}{\epsilon^2}\left(\norm{\psi}_{\infty,2}+B^{-1}\right)^2$ and $m \ge \frac{(4\sqrt{2}B\norm{\psi}_{\infty,2}+1)^2\log(2/\delta)}{\epsilon^2}$ we have $L(\bw_T) \le 4 \epsilon$ and therefore $\norm{\bw_T}^2 \le \frac{2}{\lambda} L(\bw_T) \le 8 B^2$.

    In this case, using Lemma \ref{lem:generalization_bound}, we w.p. at least $1-\delta$:
    \begin{align*}
    L_\cd(h_{\phi^{(1)},\bw^T}) &\le L_\cs(h_{\phi^{(1)}, \bw_T}) + \frac{(4\sqrt{2}B\norm{\psi}_{\infty,2}+1)\sqrt{2\log(2/\delta)}}{\sqrt{m}} + 2 B \norm{\psi-\phi^{(1)}}_{\infty,2} 
    \le 7 \epsilon
    \end{align*}
\end{proof}

\begin{lemma}
    \label{lem:good_init}
    Fix some $\delta$. Assume we initialize a network of size $r \ge 20k(2n/s)^k\log\left(\frac{2k}{\delta}\right)$. Then, w.p. at least $1-\delta$, $\W,\bb$ is $r'$-good for $r' =\frac{r}{2k}(s/2n)^k$. 
\end{lemma}

\begin{proof}
    From Lemma \ref{lem:good_neuron_prob}, the probability of drawing a good neuron is $\ge (s/2n)^k$. So, for every $i$, the probability of drawing a \emph{good} neuron with bias $\beta_i$ is at least $\frac{1}{k}(s/2n)^k$. Denote by $r'_i$ the number of good neurons with bias $\beta_i$. Observe that $\E[r'_i] = \frac{r}{2k}(s/2n)^k$. 
    Using Chernoff's bound, we have:
    \[
        \Pr\left[r'_i > \frac{r}{4k}(s/2n)^k\right] \le \exp\left(-\frac{r}{20k}(s/2n)^k\right) \le \frac{\delta}{2 k}
    \]
    and similarly $\Pr\left[r'_i < \frac{3r}{4k}(s/2n)^k\right]\le \frac{\delta}{2k}$. So, using the union bound we get the required.
\end{proof}

 \begin{theorem}
    Fix $\delta \in (0,1/2), \epsilon \in (0,1/2)$, and assume we choose:
    \begin{itemize}
        \item $s \ge \alpha^{(k)}_1 \frac{\log(1/\delta)}{\epsilon^2}$.
        \item $r =  \left\lceil \alpha^{(k)}_2(n/s)^k\log\left(1/\delta\right)\right \rceil$
        \item $m \ge \alpha^{(k)}_3 \binom{s}{k-1}n^2\log(nr/\delta)\log(1/\delta)$
        \item $T \ge \alpha^{(k)}_4\frac{\log(1/\delta)\log(T)}{\epsilon^2}$
    \end{itemize}
    for some constants $\alpha^{(k)}_1, \alpha^{(k)}_2, \alpha^{(k)}_3, \alpha^{(k)}_4$.
    Then, with probability at least $1-\delta$ over the choice of sample size and initialization, gradient descent returns after $T$ iterations a function $h$ s.t.
    $L_\cd(h) \le \epsilon$.
\end{theorem}

\begin{proof}
    Choose $r = \left\lceil 20k(2n/s)^k\log\left(\frac{2k}{\delta}\right)\right \rceil$.
    \begin{itemize}
        \item From Lemma \ref{lem:good_init}, with probability at least $1-\delta$ we get an $r'$-\emph{good} init with $r' = \frac{r}{2k}(s/2n)^k$ and notice that $10\log(2k/\delta) \le r' \le 20 \log(2k/\delta)$.
        \item Let $\eta = \frac{1}{2k}\abs{C_{k,s}}^{-1} \ge \frac{\kappa_k}{2k} \sqrt{\binom{s}{k-1}}$ and 
        $$\tau = \frac{\epsilon}{40B_kn\kappa_k\sqrt{\binom{s}{k-1}} \log(2k/\delta)} \le \min \left(\frac{\epsilon}{4B_kkr'n\eta}, \frac{1}{16\eta k n}\right)$$
        
        Choosing
        \begin{itemize}
            \item $m \ge \frac{4\cdot 40^2 B_k^2 \kappa_k n^2 \binom{s}{k-1} \log(2k/\delta)^2 \log(4nr/\delta)}{\epsilon^2} \ge \frac{4\log(4nr/\delta)}{\tau^2}$
            \item $s \ge \frac{80 kB_k^2 \log(2k/\delta)}{\epsilon^2} \ge \frac{4kr'B_k^2}{\epsilon^2}$
        \end{itemize}
        from Lemma \ref{lem:first_step}, the conditions for Lemma \ref{lem:full_one_step} are satisfied with
        \begin{enumerate}
            \item $B = B_k$.
            \item $\norm{\psi}_{\infty,2} \le \sqrt{8kr'} \le 4\sqrt{10k\log(2k/\delta)}$
            \item  
            $\norm{\psi-\phi^{(1)}}_{\infty,2} \le 4kr'n \eta \tau \le \frac{\epsilon}{B_k}$
        \end{enumerate}
        \item Choosing 
        \begin{itemize}
            \item $\frac{T}{\log(T)} \ge \frac{100\left(4\sqrt{10k\log(2k/\delta)}+B^{-1}\right)^2}{\epsilon^2} \ge \frac{100}{\epsilon^2}\left(\norm{\psi}_{\infty,2}+B^{-1}\right)^2$
            \item $m \ge \frac{(16\sqrt{2}B\sqrt{10k\log(2k/\delta)}+1)^2\log(2/\delta)}{\epsilon^2} \ge \frac{(4\sqrt{2}B\norm{\psi}_{\infty,2}+1)^2\log(2/\delta)}{\epsilon^2} $
        \end{itemize}
        using Lemma \ref{lem:full_one_step} we get w.p. at least $1-2\delta$ we have $L_\cd(h) \le 7 \epsilon$.
    \end{itemize}

\end{proof}

\subsection{Feature selection with an under-sparse initialization and a narrow network}
\label{subsec:under-feature-selection-proof}
Fix some subset $S \subseteq \binom{n}{k}$ to be the true parity function. Let $k$ be even. 

We train the following network:
\[
f_{\W, \bb, \bu, \beta}(\x) = \sum_{i=1}^r u_i \sigma\left(\inner{\bw_i, \x} + b_i\right) + \beta.
\]
We initialize the network as follows:
\begin{itemize}
    \item Randomly initialize $\bw_1, \dots, \bw_{r/2} \in \{\epsilon,1\}^n$ s.t. $s$ coordinates have weight 1 and rest have weight $\epsilon < \frac{1}{(n-s)}$, with a uniform distribution over all $\binom{n}{s}$ subsets.
    \item Randomly initialize $b_1, \dots, b_{r/2} \sim \left\lbrace-\epsilon\frac{k-1}{2k}, -\epsilon\frac{k-3}{2k}, \dots, -\epsilon\frac{1}{2k}, \epsilon\frac{1}{2k}, \dots, \epsilon\frac{k-3}{2k}, \epsilon\frac{k-1}{2k}\right\rbrace$.
    \item Randomly initialize $u_1, \dots, u_{r/2} \in \{\pm 1\}$ uniformly at random.
    \item Initialize $\bw_{r/2+1}, \dots, \bw_r, b_{r/2+1}, \dots, b_r$ s.t. $\bw_i = \bw_{i-r/2}$ and $b_i = b_{i-r/2}$ (symmetric initialization).
    \item Initialize $u_{r/2+1}, \dots, u_{r}$ s.t. $u_i = -u_{i-r/2}$.
    \item Initialize $\beta = 0$
\end{itemize}

Similar to the over-sparse case, we consider hinge-loss. We consider one-step of gradient descent on a sample $\cs$ with $\ell_2$ regularization (weight decay):
\[
\theta^{(1)} = (1-\lambda)\theta^{(0)} - \eta ~\mathrm{trunc}(\nabla L_\cs(f_{\theta^{(0)}}), \gamma)
\]
with learning rate $\eta$, truncation parameter $\gamma$, and the weight decay $\lambda$ chosen differently for each layer and separately for the weights and biases. Note that truncation just zeros out gradients with magnitude $\le \gamma$.

\paragraph{Majority and Half.}
We will make use of two Boolean functions: (1) Majority, and (2) Half (derivative of Majority). For input $x \in \{\pm 1\}^n$, we define
\[
\MAJ(x) = \sgn\left(\sum_{i=1}^n x_i\right)
\]
where $\sgn(a) = 1$ is $a> 0$ else $-1$. The derivative of the Majority function is denoted by $D_n \MAJ_n = \HALF$ and is defined for input $x \in \{\pm 1\}^{n-1}$ as:
\[
\HALF(x) = \mathbbm{1}\left(\sum_{i=1}^n x_i = 0\right).
\]
The corresponding Fourier coefficients corresponding to set $S$ are denoted by $\widehat{\MAJ}_n(S)$ and $\widehat{\HALF}_{n-1}(S)$. Note that both functions are permutation invariant, so the Fourier coefficients only depend on the size of the set.

\begin{lemma}[\cite{o2014analysis}]\label{lem:half}
For any integers $m \ge j$, we have
$$
    \widehat{\HALF}_{2m}(2j) = \widehat{\MAJ}_{2m + 1}(2j + 1) = (-1)^j\frac{\binom{m}{j}}{{\binom{2m}{2j}}} \frac{\binom{2m}{m}}{2^{2m}}.
$$
\end{lemma}

\paragraph{Population gradients at initialization.} Consider a fixed weight $\bw$ with the set of 1 weights indicated by $S' \subseteq S$. We first start with computing the population gradients for this neuron ($h_\bw(\x) = \sigma(\inner{\bw, \x})$ be a single ReLU neuron where $\sigma(x) = \max \{x,0\}$) at initialization for varying $S \cap S' = \bar{S}$ and parity and non-parity variables.
\begin{lemma}[Population gradient at initialization for parity variables] \label{lem:pop_parity}
Assuming $s<k$ with $s,k$ even, for $i \in S$, we have
\[
\E_{\x \sim \{\pm 1\}^n} \left[ \frac{\partial}{\partial w_i}h_\bw(\x) \cdot \chi_S(\x) \right] = \frac{1}{2}\widehat{\HALF}_{s}(\bar{S} \setminus \{i\})\widehat{\MAJ}_{n-s}(S \setminus (\bar{S}  \cup \{i\})).
\]
\end{lemma}
\begin{proof} For $i \in S$, we have
\begin{align*}
    &\E_{\x \sim \{\pm 1\}^n} \left[ \frac{\partial}{\partial w_i}h_\bw(\x) \cdot \chi_S(\x) \right]\\
        &= \E_{\x \sim \{\pm 1\}^n} \left[\sigma'(\inner{\bw,\x}) \cdot x_i \cdot \chi_S(\x) \right] \\
\intertext{Since $i \in S$:}
        &= \E_{\x \sim \{\pm 1\}^n} \left[\ind_{\left\lbrace\sum_{j \in S'}x_j + \epsilon \sum_{j \not\in S'}x_j > 0 \right\rbrace} \chi_{S \setminus \{i\}}(\x)\right] \\
\intertext{Splitting based on value of $\sum_{j \in S'}x_j$:}
        &= \E_{\x \sim \{\pm 1\}^n} \left[\ind_{\left\lbrace\sum_{j \in S'}x_j = 0 \right\rbrace}\ind_{\left\lbrace \sum_{j \not\in S'}x_j > 0 \right\rbrace} \chi_{S \setminus \{i\}}(\x)\right] + \E_{\x \sim \{\pm 1\}^n} \left[\ind_{\left\lbrace\sum_{j \in S'}x_j > 0 \right\rbrace}\ind_{\left\lbrace\sum_{j \in S'}x_j + \epsilon\sum_{j \not\in S'}x_j > 0\right\rbrace} \chi_{S \setminus \{i\}}(\x)\right]\\
\intertext{Using the fact that $\left|\epsilon\sum_{j \not\in S'}x_j\right| \le (n - s)\epsilon < 1$:}
        &= \E_{\x \sim \{\pm 1\}^n} \left[\ind_{\left\lbrace\sum_{j \in S'}x_j = 0 \right\rbrace}\ind_{\left\lbrace \sum_{j \not\in S'}x_j \ge 0 \right\rbrace} \chi_{S \setminus \{i\}}(\x)\right] + \E_{\x \sim \{\pm 1\}^n} \left[\ind_{\left\lbrace\sum_{j \in S'}x_j > 0\right\rbrace} \chi_{S \setminus \{i\}}(\x)\right]\\
\intertext{Splitting between variables in $S'$ and outside $S'$:}
       &= \E_{\x \sim \{\pm 1\}^n} \left[\ind_{\left\lbrace\sum_{j \in S'}x_j = 0 \right\rbrace} \chi_{\bar{S} \setminus \{i\}}(\x)\right]\E_{\x \sim \{\pm 1\}^n} \left[\ind_{\left\lbrace \sum_{j \not\in S'}x_j \ge 0 \right\rbrace} \chi_{S \setminus (\bar{S}  \cup \{i\})}(\x)\right] + \E_{\x \sim \{\pm 1\}^n} \left[\ind_{\left\lbrace\sum_{j \in S'}x_j > 0\right\rbrace} \chi_{S \setminus \{i\}}(\x)\right]\\
\intertext{Replacing indicators with Maj and Half appropriately:}
        &= \E_{\x \sim \{\pm 1\}^{n}} \left[\HALF_{s}(\x_{S'}) \chi_{\bar{S} \setminus \{i\}}(\x)\right]\E_{\x \sim \{\pm 1\}^{n}} \left[\frac{1}{2} \left(\MAJ_{n-s}(\x_{[n] \setminus S'}) + 1\right) \chi_{S \setminus (\bar{S}  \cup \{i\})}(\x)\right] \\
        &\qquad+ \E_{\x \sim \{\pm 1\}^{n}} \left[\frac{1}{2} \left(\MAJ_{s}(\x_{S'}) + 1\right) \chi_{\bar{S} \setminus \{i\}}(\x)\right]  \E_{\x \sim \{\pm 1\}^{n - s}} \left[\chi_{S \setminus (\bar{S}  \cup \{i\})}(\x)\right]\\
\intertext{Replacing indicators with Fourier coefficients appropriately:}
        &= \frac{1}{2}\widehat{\HALF}_{s}(\bar{S} \setminus \{i\}) \left(\widehat{\MAJ}_{n-s}(S \setminus (\bar{S}  \cup \{i\})) + \ind_{\left\lbrace S \subseteq \bar{S}  \cup \{i\} \right\rbrace}\right)  +\frac{1}{2} \left(\widehat{\MAJ}_{s}(\bar{S} \setminus \{i\}) + \ind_{\left\lbrace \bar{S} \setminus \{i\} = \phi \right\rbrace}\right)\ind_{\left\lbrace S \subseteq \bar{S}  \cup \{i\} \right\rbrace}\\
        &= \frac{1}{2}\widehat{\HALF}_{s}(\bar{S} \setminus \{i\})\widehat{\MAJ}_{n-s}(S \setminus (\bar{S}  \cup \{i\}))+ \frac{1}{2} \left(\widehat{\MAJ}_{s}(\bar{S} \setminus \{i\}) + \ind_{\left\lbrace \bar{S} \setminus \{i\} = \phi \right\rbrace} + \widehat{\HALF}_{s}(\bar{S} \setminus \{i\}) \right)  \ind_{\left\lbrace S \subseteq \bar{S}  \cup \{i\} \right\rbrace}\\
\intertext{Since $k, s$ are even and $k < s$, $|\bar{S}\cup \{i\}| \le s + 1 < k = |S|$ therefore $\ind_{\left\lbrace S \subseteq \bar{S}  \cup \{i\} \right\rbrace} = 0$:}
        &= \frac{1}{2}\widehat{\HALF}_{s}(\bar{S} \setminus \{i\})\widehat{\MAJ}_{n-s}(S \setminus (\bar{S}  \cup \{i\})).
\end{align*}
This gives us the desired result.
\end{proof}

\begin{lemma}[Population gradient at initialization for non-parity variables] \label{lem:pop_non_parity}
Assuming $s<k$ with $s,k$ even, for $i \not\in S$, we have
\[
\E_{\x \sim \{\pm 1\}^n} \left[ \frac{\partial}{\partial w_i}h_\bw(\x) \cdot \chi_S(\x) \right] = \frac{1}{2}\widehat{\HALF}_{s}(\bar{S} \cup (S' \cap \{i\})) \widehat{\MAJ}_{n-s}((S \cup \{i\}) \setminus S').
\]
\end{lemma}
\begin{proof} 
For $i \not\in S$, using similar calculations as in the proof of Lemma \ref{lem:pop_parity}, we have
\begin{align*}
    &\E_{\x \sim \{\pm 1\}^n} \left[ \frac{\partial}{\partial w_i}h_\bw(\x) \cdot \chi_S(\x) \right]\\
     &= \E_{\x \sim \{\pm 1\}^n} \left[\sigma'(\inner{\bw,\x}) \cdot x_i \cdot \chi_S(\x) \right] \\
        &= \E_{\x \sim \{\pm 1\}^n} \left[\ind_{\left\lbrace\sum_{j \in S'}x_j + \epsilon \sum_{j \not\in S'}x_j > 0 \right\rbrace} \chi_{S \cup \{i\}}(\x)\right] \\
        &= \E_{\x \sim \{\pm 1\}^n} \left[\ind_{\left\lbrace\sum_{j \in S'}x_j = 0 \right\rbrace}\ind_{\left\lbrace \sum_{j \not\in S'}x_j > 0 \right\rbrace} \chi_{S \cup \{i\}}(\x)\right] + \E_{\x \sim \{\pm 1\}^n} \left[\ind_{\left\lbrace\sum_{j \in S'}x_j > 0\right\rbrace} \chi_{S \cup \{i\}}(\x)\right]\\
        &= \E_{\x \sim \{\pm 1\}^n} \left[\ind_{\left\lbrace\sum_{j \in S'}x_j = 0 \right\rbrace} \chi_{\bar{S} \cup (S' \cap \{i\})}(\x)\right]\E_{\x \sim \{\pm 1\}^n} \left[\ind_{\left\lbrace \sum_{j \not\in S'}x_j > 0 \right\rbrace} \chi_{(S \cup \{i\}) \setminus S' }(\x)\right] \\
        &\qquad+ \E_{\x \sim \{\pm 1\}^n} \left[\ind_{\left\lbrace\sum_{j \in S'}x_j > 0\right\rbrace} \chi_{S \cup \{i\}}(\x)\right]\\
        &= \E_{\x \sim \{\pm 1\}^{n}} \left[\HALF_{s}(\x_{S'}) \chi_{\bar{S} \cup (S' \cap \{i\})}(\x)\right]\E_{\x \sim \{\pm 1\}^{n}} \left[\frac{1}{2} \left(\MAJ_{n-s}(\x_{[n] \setminus S'}) + 1\right) \chi_{(S \cup \{i\}) \setminus S'}(\x)\right] \\
        &\qquad+ \E_{\x \sim \{\pm 1\}^{n}} \left[\frac{1}{2} \left(\MAJ_{s}(\x_{S'}) + 1\right)\chi_{\bar{S} \cup (S' \cap \{i\})}(\x)\right]  \E_{\x \sim \{\pm 1\}^{n}} \left[\chi_{(S \cup \{i\}) \setminus S'}(\x)\right]\\
        &= \frac{1}{2}\widehat{\HALF}_{s}(\bar{S} \cup (S' \cap \{i\})) \widehat{\MAJ}_{n-s}((S \cup \{i\}) \setminus S')  \\
        &\qquad+ \frac{1}{2}\left( \widehat{\MAJ}_{s}(\bar{S} \cup (S' \cap \{i\})) + \ind_{\left\lbrace \bar{S} = \phi \land i \not \in S' \right\rbrace} + \widehat{\HALF}_{s}(\bar{S} \setminus \{i\}) \right)  \ind_{\left\lbrace S \cup \{i\} \subseteq S'\right\rbrace}\\
        \intertext{Using the fact that $|S'| < |S|$ therefore $\ind_{\left\lbrace S \cup \{i\} \subseteq S'\right\rbrace} = 0$:}
        &= \frac{1}{2}\widehat{\HALF}_{s}(\bar{S} \cup (S' \cap \{i\})) \widehat{\MAJ}_{n-s}((S \cup \{i\}) \setminus S').
\end{align*}
\end{proof}

Similar to the over-sparse setting, we say a neuron is \textit{good} if $S' \subset S$, that is, if the selected variables are a subset of the relevant variables. Then we have,
\begin{lemma}[Population gradient at initialization for good neurons] \label{lem:good_neurons}
Assuming $s<k$ with $s,k$ even, for good neurons, for $\bar{S} = S' \cap S$, we have
\begin{align*}
\E_{\x \sim \{\pm 1\}^n} \left[ \frac{\partial}{\partial w_i}h_\bw(\x) \cdot \chi_S(\x) \right]= 
\begin{cases}
\frac{1}{2}\widehat{\HALF}_{s}(s)\widehat{\MAJ}_{n-s}(k - s - 1) & \text{ if }i \in S \setminus S',\\
\frac{1}{2}\widehat{\HALF}_{s}(s) \widehat{\MAJ}_{n-s}(k - s + 1) & \text{ if }i \not\in S,\\
0 &\text{ otherwise.}
\end{cases}
\end{align*}
\end{lemma}
\begin{proof}
Using Lemma \ref{lem:pop_parity} and \ref{lem:pop_non_parity}, we get \begin{enumerate}
    \item $i \in S, i \in S'$: 
    \[
    \E_{\x \sim \{\pm 1\}^n} \left[ \frac{\partial}{\partial w_i}h_\bw(\x) \cdot \chi_S(\x) \right]= \frac{1}{2}\widehat{\HALF}_{s}(\bar{S} \setminus\{i\})\widehat{\MAJ}_{n-s}(S \setminus \bar{S})
    \]
    \item $i \in S, i \not \in S'$: 
    \[
    \E_{\x \sim \{\pm 1\}^n} \left[ \frac{\partial}{\partial w_i}h_\bw(\x) \cdot \chi_S(\x) \right]= \frac{1}{2}\widehat{\HALF}_{s}(\bar{S})\widehat{\MAJ}_{n-s}(S \setminus (\bar{S}  \cup \{i\}))
    \]
    \item $i \not\in S, i \in S'$: 
    \[
    \E_{\x \sim \{\pm 1\}^n} \left[ \frac{\partial}{\partial w_i}h_\bw(\x) \cdot \chi_S(\x) \right]= \frac{1}{2}\widehat{\HALF}_{s}(\bar{S} \cup \{i\}) \widehat{\MAJ}_{n-s}(S \setminus \bar{S})
    \]
    \item $i \not\in S, i \not\in S'$: 
    \[
    \E_{\x \sim \{\pm 1\}^n} \left[ \frac{\partial}{\partial w_i}h_\bw(\x) \cdot \chi_S(\x) \right]= \frac{1}{2}\widehat{\HALF}_{s}(\bar{S}) \widehat{\MAJ}_{n-s}((S \setminus \bar{S}) \cup \{i\})
    \]
\end{enumerate}
Since $S' \subseteq S$, the arguments to the $\widehat{\HALF}$ for (1) and (3) are odd. The corresponding Fourier coefficients for odd sets for $\HALF$ are 0 (see \cite{o2014analysis}), thus we have the desired result.
\end{proof}

\begin{theorem}[Formal version of Theorem \ref{thm:undersparse}]   %
    Fix $s,k,\epsilon$ such that $s < k, \epsilon > 0$. Then for network width $r \geq \Omega((n/k)^s)$, the initialization scheme proposed here guarantees that for every $(n,k)$-parity distribution $\cd$, with probability at least $0.99$ over the choice of sample and initialization, after one step of batch gradient descent with sample size $m = O((n/k)^{k-s-1})$ and appropriate choice of learning rate, there is a subnetwork in the one-layer ReLU MLP that has at least $\frac{1}{2} - \epsilon$ correlation with the parity function.
\end{theorem}
\begin{proof}
To compute the parity function, we need to have $k$ neurons which identify the correct coordinates and have the appropriate biases. In order to identify the correct coordinates, we will focus only on good neurons, and on population gradient. We will then argue by standard concentration arguments that this holds from samples.

Let us consider a good neuron with $S' \subseteq S$. Firstly note that the scale of the bias is set such that it does not affect the gradient, Since it is at most $\epsilon/2$. Thus we can assume the no bias case for gradient computation. For all $i \in S \setminus S'$, that is, the set of relevant variables that are not selected in the initialization, let $\xi_{S \setminus S'}$ denote the gradient at initialization, and for all $i \not \in S$, that is, the set of irrelevant variables, let $\xi_{[n] \setminus S}$ denote the gradient at initialization. Then using Lemma \ref{lem:good_neurons} and Lemma \ref{lem:half}, we have
\[
\frac{\left|\xi_{S \setminus S'}\right|}{\left|\xi_{[n] \setminus S}\right|} =  \frac{n - s}{k - s - 1}  > 1.
\]
With $\eta$ and $\gamma$ being 0 on the bias terms and the second layer, $\eta = -\frac{\epsilon}{2k\left|\xi_{S \setminus S'}\right|}$ for the first layer weights, $\lambda = 1 - \frac{\epsilon}{2k}$ and $\gamma = \left|\xi_{[n] \setminus S}\right|$ , one step of truncated population gradient descent gives us, for all $i \in S$, $w_i = \frac{\epsilon}{2k}$ and for all $i \not \in S$, $|w_i| = \frac{\epsilon^2}{k}$. Since our initialization has two copies of the same neuron with second layer weights $1$ and $-1$, one of them will have the gradient in the correct direction, ensuring the above, in particular, the one with the weight being $\sgn(\xi_{S \setminus S'})$. Since $\epsilon$ can be set to be arbitrarily small, terms with $\epsilon^2$ can be ignored in comparison to terms with $\epsilon$. Thus the non-parity coefficients do not affect the output of the function and we get the appropriate parity coefficients (scaled by $\epsilon/2k$). To extend these guarantees to the batch gradient setting, we need to compute gradients to precision 
\[
\tau = \left|\xi_{[n] \setminus S}\right|/2 = c_s (n - s)^{-\frac{k - s - 1}{2}}
\]
for some constant $c_s$ dependent only on $s$. This implies a sample complexity of $O(c_s^2(n - s)^{k - s - 1})$ using standard Chernoff bound. As for the width, we need to ensure that we have the required number of good neurons with appropriate bias. The probability of a randomly initialized neuron to be good is 
\[
\frac{\binom{k}{s}}{\binom{n}{s}} = \Theta\left(\left(\frac{k}{n}\right)^s\right).
\]
The probability of choosing the appropriate bias is $1/(k+1)$. Thus to be able to choose $k+1$ good neurons with correct biases, we need width $O(k^2(n/k)^s)$.
\end{proof}

We provide some additional remarks on the proof of Theorem~\ref{thm:undersparse}:
\begin{itemize}
    \item The sample complexity compared to the dense initialization studied by \cite{barak2022hidden} improves by a factor of $n^s$, at the cost of a higher width by a factor of $n^s$.
    \item We conjecture that Theorem~\ref{thm:undersparse} can be strengthened to an end-to-end guarantee for the full network, like Theorem~\ref{thm:oversparse_upper_bound}. The technical challenge lies in analyzing how weight decay uniformly prunes the large irrelevant coordinates, without decaying the good subnetwork. We believe that this occurs robustly (from the experiments on sparse initialization), but requires a more refined analysis of the optimization trajectory.
\end{itemize}

\section{Full experimental results}
\label{app-sec:experiments}

\subsection{Full sweeps over dataset size and width}
\label{subsec:data-width-experiments}

\begin{figure}
    \centering
    \includegraphics[width=\linewidth]{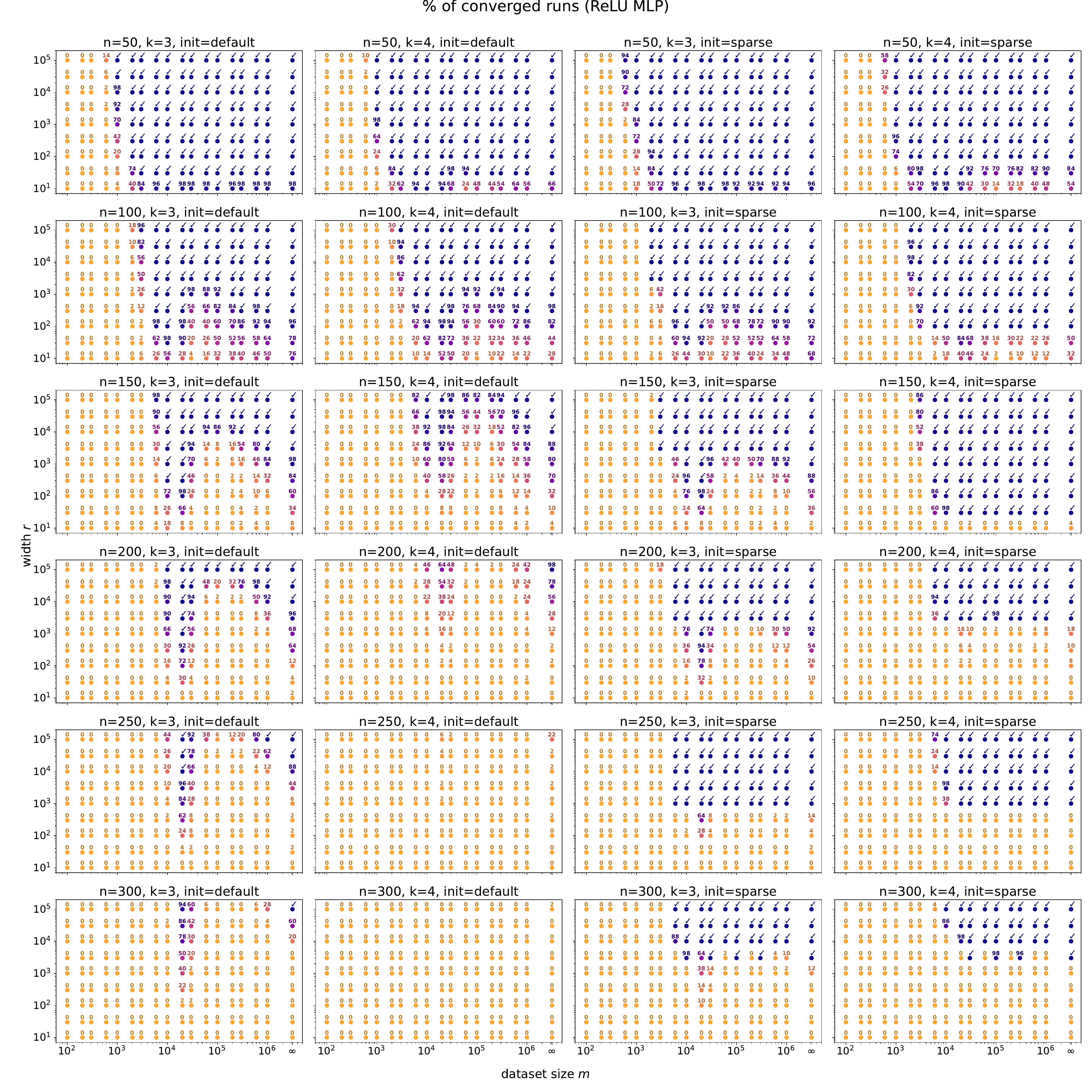}
    \caption{Full results for the sweep over dataset sizes and MLP widths. Each point represents $50$ independent training runs of a 2-layer ReLU MLP on the offline sparse parity problem (input dimension $n$, parity degree $k$, initialization scheme, width $r$, dataset size $m$). All other algorithmic choices (learning rate $\eta = 0.1$, weight decay $\lambda = 0.01$, batch size $B = 32$, number of training iterations $T = 10^5$) are kept the same. We note the following: (1) \textbf{A success frontier}, where wide networks can learn at small sample sizes $m \ll BT$ (far outside online regime); (2) \textbf{Monotonic benefit of width}: overparameterization does not worsen the failure mode of overfitting in this setting, and amplifies success probabilities; (3) \textbf{Benefit of sparse axis-aligned initialization}: for sufficiently large $n$ where default initialization no longer works, sparse initialization scheme enlarges the feasible regime; (4) \textbf{Non-monotonic effects of dataset size:} there are unpredictable failures as we vary $m$ (horizontal slices of these plots). }
    \label{fig:full_sweep}
\end{figure}

In our main set of large-scale synthetic experiments, we train a large number of 2-layer MLPs to solve various $(n,k)$-sparse parity problems, from $m$ samples. The full set of hyperparameters is listed below:
\begin{itemize}
    \item Problem instance sizes: $n \in \{50, 100, 150, \ldots, 300\}$, $k \in \{3,4\}$. Note that \citet{barak2022hidden} investigate empirical \emph{computational time} scaling curves for larger $k$ (up to $8$) in the online ($m=\infty$) setting. We are able to observe convergence for larger $k$ in the offline setting, but we omit these results from the systematic grid sweep (the feasible regime is too small in terms of $n$).
    \item Dataset size: $m \in \{100, 200, 300, 600, 1000, 2000, 3000, \ldots, 600000, 1000000\}$. We also include runs in the online regime (rightmost columns in Figure~\ref{fig:full_sweep}), which correspond to the regime studied by \citet{barak2022hidden}.
    \item Network width: $r \in \{10, 30, 100, 300, \ldots, 10000, 30000, 100000\}$.
    \item Initialization scheme: PyTorch default (uniform on the interval $[-1/\sqrt{n}, 1/\sqrt{n}]$), and random $2$-sparse rows. In coarser-grained hyperparameter searches, we found $2$ to be the optimal sparsity constant for large-width ($\geq 1000$) regimes studied in this paper; we do not fully understand why this is the case. We also keep the PyTorch default initialization scheme (uniform on the width-dependent interval $[-1/\sqrt{r}, 1/\sqrt{r}]$) for the second layer, and use default-initialized biases.
\end{itemize}

At each point in this hyperparameter space, we conduct $50$ training runs, and record the \emph{success probability}, defined as the probability of achieving test error $\leq 10\%$ on a held-out sample of size $10^4$ within $T = 10^5$ training iterations. The hyperparameters for SGD, selected via coarse-grained hyperparameter search to optimize for convergence time in the $n=200, k=3, r=10000$ setting, are as follows: minibatch size $B = 32$; learning rate $\eta = 0.1$; weight decay $\lambda = 0.01$.

Figure~\ref{fig:full_sweep} summarizes all of our runs, and overviews all of the findings (1) through (4) enumerated in the main paper. We go into more detail below:

\begin{enumerate}
    \item[(1)] \textbf{A ``success frontier'': large width can compensate for small datasets.} We observe convergence and perfect generalization when $m \ll n^k$. In such regimes, which are far outside the online setting considered by \citet{barak2022hidden}, high-probability sample-efficient learning is enabled by large width. Note that neither our theoretical or empirical results have sufficient granularity to predict or measure the precise way the smallest feasible sample size $m$ scales with the other size parameters (like $n,k,r$). The theoretical upper bounds show that if $r = \Omega(n^k)$, idealized algorithms (modified for tractability of analysis) can obtain $O(\poly(n))$ or even $O(\log(n))$ sample complexity, and smaller $r$ can yield milder reductions of the exponent.
    \item[(2)] \textbf{Width is monotonically beneficial, and buys data, time, and luck.} Despite the capacity of wider neural networks to overfit larger datasets, we find that there are \emph{monotonic sample-efficiency benefits} to increasing network width, in \emph{all} of the hyperparameter settings considered in the grid sweep. This can be quickly quantitatively confirmed by starting at any point in Figure~\ref{fig:full_sweep}, and noting that success probabilities only increase\footnote{Small exceptions (such as $98\% \rightarrow 96\%$ for $n=100, k=4, m=2000$) are all within the standard error margins of Bernoulli confidence intervals.} going upwards (increasing $r$, keeping all other parameters equal). Along some of these vertical slices, we observe that transitions from $0\%$ to $100\%$ are present: at these corresponding dataset sizes $m$, \emph{large width makes sample-efficient learning possible}. Figure~\ref{fig:main-detail} (center) shows this in greater detail, by choosing a denser grid of sample sizes $m$ near the empirical statistical limit.
    \item[(3)] \textbf{Sparse axis-aligned initialization buys data, time, and luck.} Used in conjunction with a wide network, we observe that a sparse, axis-aligned initialization scheme yields strong improvements on all of these axes. This can be seen by comparing the pairs of subplots in columns 1 vs. 3 and 2 vs. 4 in Figure~\ref{fig:full_sweep}. We found that $s = 2$ (i.e. initialize every hidden-layer neuron with a random $2$-hot weight vector) works best for the settings considered in this study.
    \item[(4)] \textbf{Intriguing effects of dataset size.} Unlike the monotonicity along vertical slices in Figure~\ref{fig:full_sweep}, some of the \emph{horizontal} slices exhibit non-monotonic success probabilities. Namely, as $m$ increases, keeping all else the same, the network enters and exits a first feasible regime; then, at large enough sample sizes (including the online setting), learning is observed to succeed again. Sparse initialization reduces this counterintuitive behavior, but not entirely (see, e.g., the $n=200, k=3$ cell). We do not attempt to explain this phenomenon; however, we found in preliminary investigations that the locations of the transitions are sensitive to the choice of weight decay hyperparameter. Figure~\ref{fig:main-detail} (right) shows this in greater detail, plotting median convergence times (as defined above) instead of success probabilities.
\end{enumerate}

\subsection{Lottery ticket subnetworks}\label{subsec:lottery}
Our theoretical analysis of sparse networks and experimental findings suggest that width provides a form of parallelization: wider networks have a higher probability of containing lucky neurons which have sufficient Fourier gaps at initialization to learn from the dataset. In Figure~\ref{fig:lottery} we perform an experiment in the style of \cite{frankle2018lottery}, showing that indeed the neurons which end up being important in a wide sparsely-initialized network form an unusually lucky `winning ticket' subnetwork. When we rewind the weights of this subnetwork to initialization, its test error starts out poor, but when we train just this subnetwork from initialization its performance quickly improves, unlike the large majority of randomly initialized subnetworks of the same size. For this experiment, batch size=32 and learning rate=0.1.

\begin{figure}
    \centering
    \includegraphics[width=0.9\linewidth]{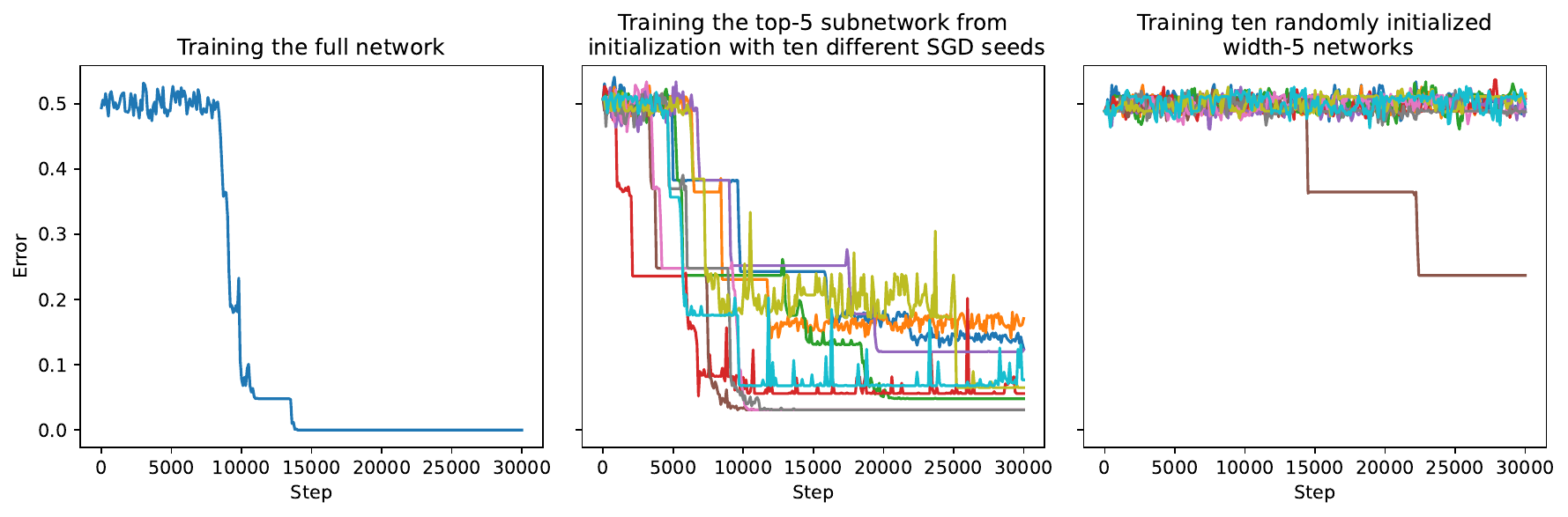}
    \caption{(Lottery tickets) Left: Training a width-100 MLP on the (n=50, k=5)-online sparse parity task, where each hidden neuron initialized with 2 non-zero incoming weights. Center: We prune all but the top 5 neurons by weight norm at the end of training; rewind weights to the original initialization, and retrain with various SGD random seeds. Right: The same as Center, but the weights are randomly reinitialized in each run.}
    \label{fig:lottery}
\end{figure}

\subsection{Training wide \& sparsely-initialized MLPs on natural tabular data}
\label{subsec:tabular-experiments}

\begin{figure}
    \centering
    \includegraphics[width=0.98\linewidth]{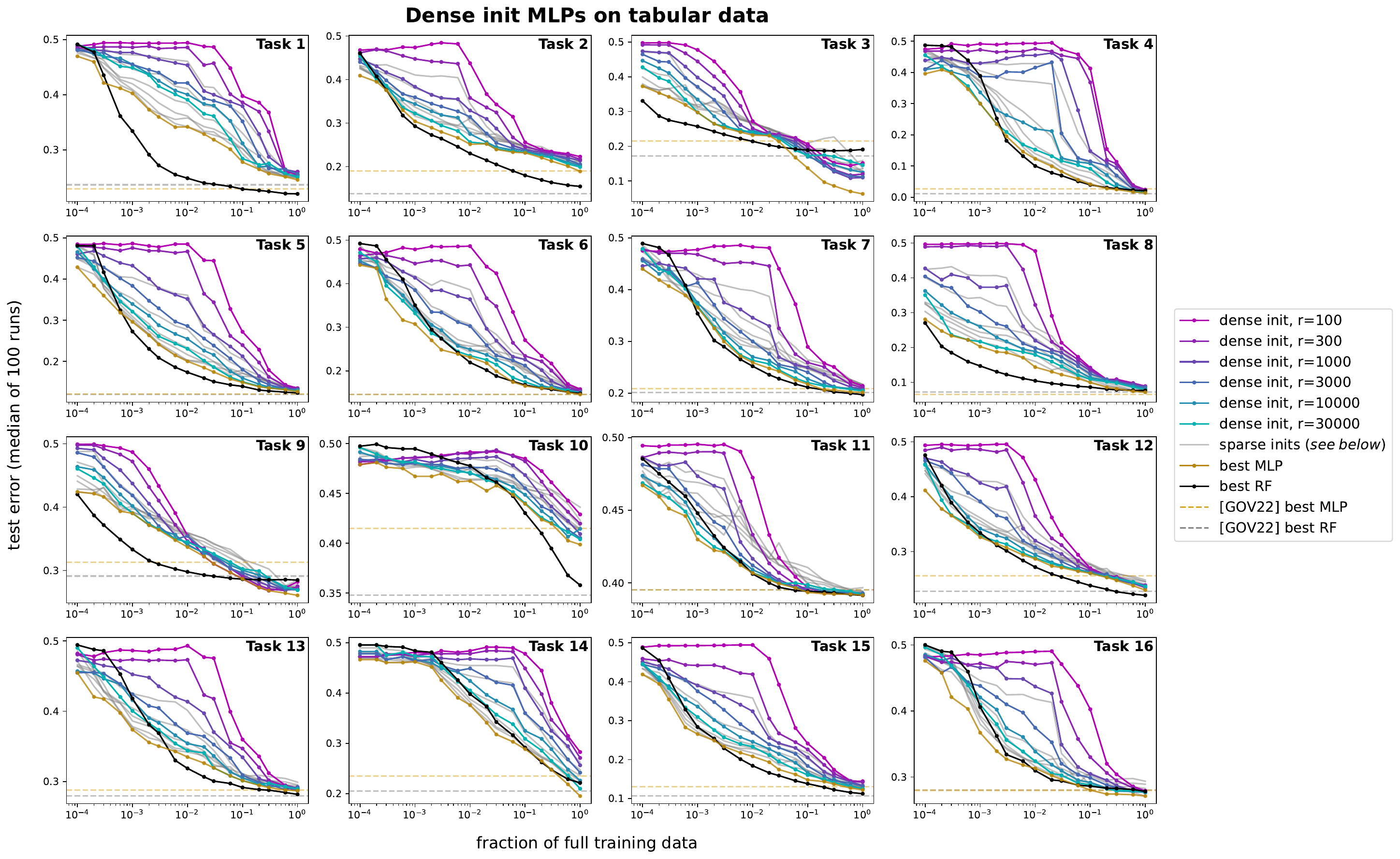}
    \includegraphics[width=0.98\linewidth]{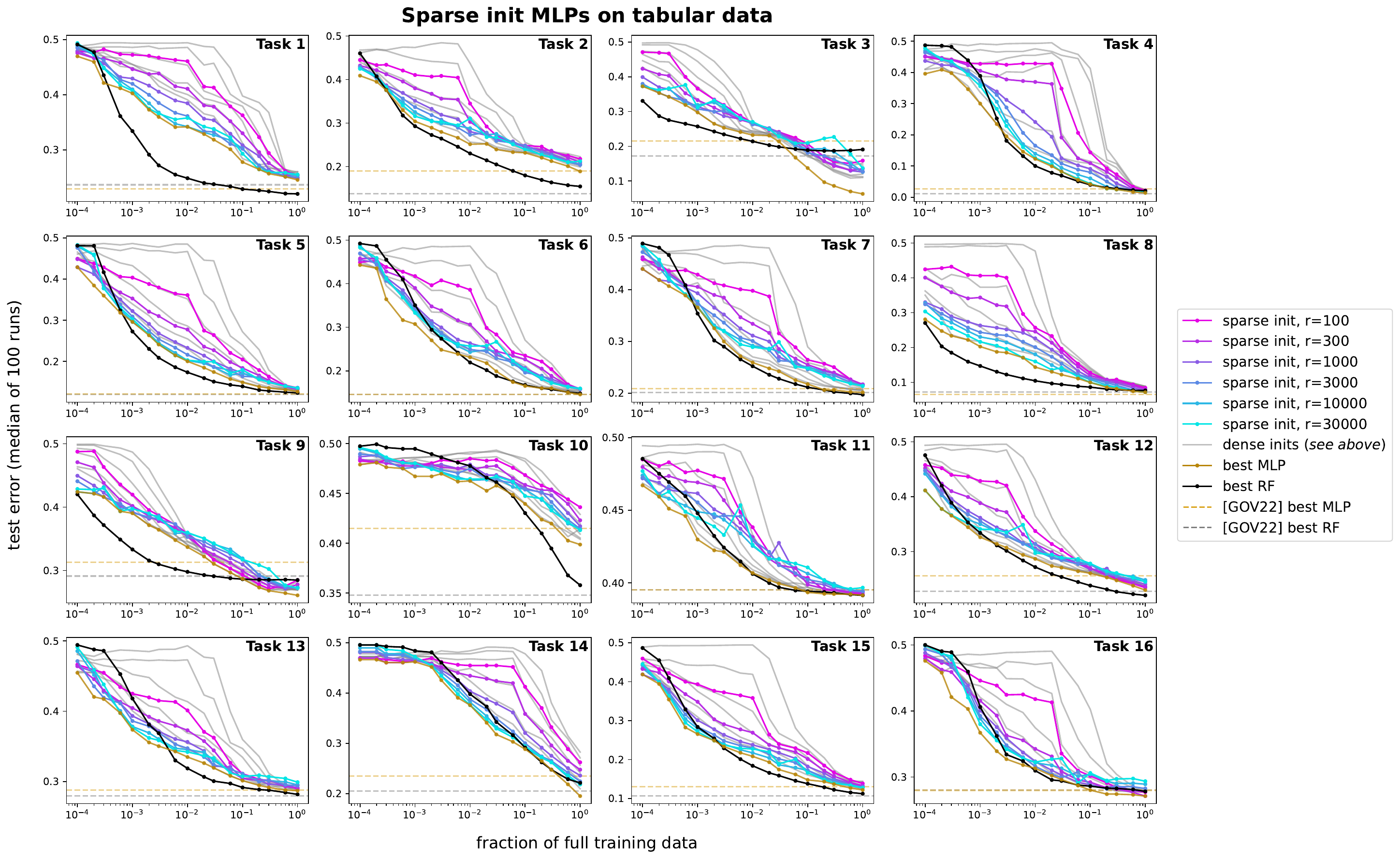}
    \caption{Full results of our tabular data experiments (following \citet{grinsztajn2022why}), varying MLP width $r$, sample size $m$ (via downsampling the training data), and initialization sparsity. Both width and sparsity tend to improve generalization, especially for small datasets. See Tables~\ref{tab:tabular-meta} and \ref{tab:tabular-results} for further details.}
    \label{fig:tabular-full}
\end{figure}

As a preliminary investigation of whether our findings translate well to realistic settings, we conduct experiments on the tabular benchmark curated by \cite{grinsztajn2022why}. For simplicity, we use all 16 numerical classification tasks from the benchmark. These datasets originate from diverse domains such as health care, finance, and experimental physics. Our main comparison is between MLPs (with algorithmic choices inspired by our sparse parity findings) and random forests \cite{breiman2001random}. We chose random forests as a baseline because they are known to achieve competitive performance with little tuning.
Figure~\ref{fig:tabular-full} summarizes our results, in which we find that wide and sparsely-initialized MLPs improve sample efficiency on small tabular datasets. We describe these experiments in full detail below.

\paragraph{Data preprocessing.}
We standardize the dataset, centering each feature to have mean 0 and normalizing each feature to have standard deviation 1.\footnote{Note that \cite{grinsztajn2022why} instead transform each feature such that its empirical marginal distribution approximates a standard Gaussian.}
For each task, we set aside 10\% of the data for the test set, and downsample varying fractions of the remaining data to form a training set. We vary the downsampling fraction on an exponential grid: \{1, 0.6, 0.3, 0.2, 0.1, 0.06, 0.03, 0.02, 0.01, 0.006, 0.003, 0.002, 0.001, 0.0006, 0.0003, 0.0002, 0.0001\}. In each of 100 i.i.d. trials for each setting, we re-randomize the validation split as well as the downsampled training set.

\paragraph{Algorithms.} MLPs are trained using the same architectural choices noted in Section~\ref{subsec:data-width-experiments}, except the hyperparameters noted below.
We use the scikit-learn \cite{scikit-learn} RandomForestClassifier. Hyperparameters not mentioned are set to library defaults.

\paragraph{Hyperparameter choices for MLPs.}
\begin{itemize}
    \item Width $r$: $\{100, 300, 1000, 3000, 10000, 30000\}$. For deeper networks, the hidden layers are set to $r \times r$.
    \item Depth (number of hidden layers): $\{1,2,3\}$. In all of these settings, depth-$1$ networks are nearly uniformly outperformed by deeper ones of the same width.
    \item Sparsity of initialization: $s \in \{2, 4\}$, and also dense (uniform) initialization. $4$-sparse initialization is nearly uniformly outperformed by $2$-sparse (in agreement with our experiments on sparse parity).
    \item Weight decay: $\{0, 0.01\}$. In these settings, this hyperparameter had a minimal effect.
    \item Learning rate: $10^{-3}$
    \item Batch size: $256$
    \item Training epochs: $100$
\end{itemize}

\paragraph{Hyperparameter choices for RandomForestClassifier.}
\begin{itemize}
    \item \texttt{max\_depth}: $\{-1, 2, 3, 4, 10\}$
    \item \texttt{max\_features}: $\{1, 2, 3, 4, \mathrm{sqrt}, \mathrm{None}\}$
    \item \texttt{n\_estimators}: $\{10, 100, 1000\}$
\end{itemize}

\paragraph{Larger widths.} For all of the tasks except 2 and 9 (since these are significantly larger datasets), we include extra runs with even larger networks: depth-2 MLPs with non-uniform width (i.e. sequence of hidden layer dimensions) $\{(100000,10000), (10000,100000)\}$.

\paragraph{Plots in Figure~\ref{fig:tabular-full}.} For clarity of presentation, in the plots where we vary MLP width and initialization sparsity, we fix depth to $3$ and sparsity level $s=2$. Qualitative trends are similar for other settings. The gold ``best MLP'' curves show the best median-of-100 validation losses across all architectures in the search space.

\paragraph{Comparison with full-data baselines.} To ensure that our baseline algorithm choices for MLPs and random forests are reasonable, we present a comparison with the results of \cite{grinsztajn2022why} (who performed extensive hyperparameter search) on the full datasets. These are shown as the dotted lines in Figure~\ref{fig:tabular-full}, as well as Table~\ref{tab:tabular-results}.

\begin{table}[]
    \centering
    
\footnotesize
\begin{tabular}{|l|l|l|l|}
\hline
\multicolumn{1}{|c|}{\textbf{Task}} & \multicolumn{1}{c|}{\textbf{OpenML identifier}} & \multicolumn{1}{c|}{\textbf{\# features}} & \multicolumn{1}{c|}{\textbf{\# examples}} \\ \hline
1                                      & credit                                          & 11                                        & 16714                                     \\ \hline
2                                      & electricity                                     & 8                                         & 38474                                     \\ \hline
3                                      & covertype                                       & 11                                        & 566602                                    \\ \hline
4                                      & pol                                             & 27                                        & 10082                                     \\ \hline
5                                      & house\_16H                                      & 17                                        & 13488                                     \\ \hline
6                                      & MagicTelescope                                  & 11                                        & 13376                                     \\ \hline
7                                      & bank-marketing                                  & 8                                         & 10578                                     \\ \hline
8                                      & MiniBooNE                                       & 51                                        & 72998                                     \\ \hline
9                                      & Higgs                                           & 25                                        & 940160                                    \\ \hline
10                                     & eye\_movements                                  & 21                                        & 7608                                      \\ \hline
11                                     & Diabetes130US                                   & 8                                         & 71090                                     \\ \hline
12                                     & jannis                                          & 55                                        & 57580                                     \\ \hline
13                                     & default-of-credit-card-clients                  & 21                                        & 13272                                     \\ \hline
14                                     & Bioresponse                                     & 420                                       & 3434                                      \\ \hline
15                                     & california                                      & 9                                         & 20634                                     \\ \hline
16                                     & heloc                                           & 23                                        & 10000                                     \\ \hline
\end{tabular}
    \vspace{1cm}
    \caption{Metadata for the 16 tabular classification benchmarks, curated by \citet{grinsztajn2022why} (January 2023 version, benchmark suite ID 337) and publicly available via OpenML.}
    \label{tab:tabular-meta}
\end{table}

\begin{table}[]
    \centering
    
\footnotesize

\begin{tabular}{|l|ccc|ccc|lll|}
\hline
\multicolumn{1}{|c|}{\multirow{2}{*}{\textbf{Task}}} & \multicolumn{3}{c|}{\textbf{Best MLP}}                       & \multicolumn{3}{c|}{\textbf{Best RF}}                        & \multicolumn{3}{c|}{\citep{grinsztajn2022why}}                                                                   \\
\multicolumn{1}{|c|}{}                               & \textbf{full}                      &\textcolor{Purple}{\textbf{10\%}}&\textcolor{NavyBlue}{\textbf{1\%}}& \textbf{full}                      &\textcolor{Purple}{\textbf{10\%}}&\textcolor{NavyBlue}{\textbf{1\%}}& \multicolumn{1}{c}{\textbf{MLP}} & \multicolumn{1}{c}{\textbf{RF}} & \multicolumn{1}{c|}{\textbf{Best model}} \\ \hline
1                                                    & \multicolumn{1}{c|}{24.5}          & \multicolumn{1}{c|}{\textcolor{Purple}{27.8}}          &\textcolor{NavyBlue}{34.2}& \multicolumn{1}{c|}{\textbf{22.0}} & \multicolumn{1}{c|}{\textbf{\textcolor{Purple}{22.9}}} &\textcolor{NavyBlue}{\textbf{24.8}}& \multicolumn{1}{c|}{22.9} & \multicolumn{1}{c|}{23.6} & 22.5 (GBT)          \\ \hline
2                                                    & \multicolumn{1}{c|}{18.9}          & \multicolumn{1}{c|}{\textcolor{Purple}{23.2}}          &\textcolor{NavyBlue}{25.1}& \multicolumn{1}{c|}{\textbf{15.4}} & \multicolumn{1}{c|}{\textbf{\textcolor{Purple}{18.0}}} &\textcolor{NavyBlue}{\textbf{23.0}}& \multicolumn{1}{c|}{23.6} & \multicolumn{1}{c|}{13.7} & 12.8 (XGB)          \\ \hline
3                                                    & \multicolumn{1}{c|}{\textbf{6.3}}  & \multicolumn{1}{c|}{\textbf{\textcolor{Purple}{13.7}}} &\textcolor{NavyBlue}{23.2}& \multicolumn{1}{c|}{19.1}          & \multicolumn{1}{c|}{\textcolor{Purple}{18.9}}          &\textcolor{NavyBlue}{\textbf{21.4}}& \multicolumn{1}{c|}{21.6} & \multicolumn{1}{c|}{17.1} & 17.1 (RF)           \\ \hline
4                                                    & \multicolumn{1}{c|}{\textbf{1.4}}  & \multicolumn{1}{c|}{\textcolor{Purple}{4.2}}           &\textcolor{NavyBlue}{12.3}& \multicolumn{1}{c|}{2.1}           & \multicolumn{1}{c|}{\textbf{\textcolor{Purple}{4.0}}}  &\textcolor{NavyBlue}{\textbf{10.0}}& \multicolumn{1}{c|}{6.3}  & \multicolumn{1}{c|}{1.9}  & 1.7 (XGB)           \\ \hline
5                                                    & \multicolumn{1}{c|}{12.7}          & \multicolumn{1}{c|}{\textcolor{Purple}{14.9}}          &\textcolor{NavyBlue}{20.0}& \multicolumn{1}{c|}{\textbf{12.4}} & \multicolumn{1}{c|}{\textbf{\textcolor{Purple}{13.9}}} &\textcolor{NavyBlue}{\textbf{17.4}}& \multicolumn{1}{c|}{12.1} & \multicolumn{1}{c|}{12.0} & 11.1 (XGB)          \\ \hline
6                                                    & \multicolumn{1}{c|}{\textbf{14.6}} & \multicolumn{1}{c|}{\textbf{\textcolor{Purple}{16.5}}} &\textcolor{NavyBlue}{23.1}& \multicolumn{1}{c|}{14.8}          & \multicolumn{1}{c|}{\textcolor{Purple}{16.8}}          &\textcolor{NavyBlue}{\textbf{21.9}}& \multicolumn{1}{c|}{14.6} & \multicolumn{1}{c|}{14.5} & 13.9 (FTT)          \\ \hline
7                                                    & \multicolumn{1}{c|}{20.2}          & \multicolumn{1}{c|}{\textcolor{Purple}{21.8}}          &\textcolor{NavyBlue}{25.9}& \multicolumn{1}{c|}{\textbf{19.7}} & \multicolumn{1}{c|}{\textbf{\textcolor{Purple}{21.1}}} &\textcolor{NavyBlue}{\textbf{25.2}}& \multicolumn{1}{c|}{20.9} & \multicolumn{1}{c|}{20.1} & 19.5 (GBT)          \\ \hline
8                                                    & \multicolumn{1}{c|}{\textbf{7.2}}  & \multicolumn{1}{c|}{\textcolor{Purple}{10.0}}          &\textcolor{NavyBlue}{14.4}& \multicolumn{1}{c|}{7.7}           & \multicolumn{1}{c|}{\textbf{\textcolor{Purple}{8.6}}}  &\textcolor{NavyBlue}{\textbf{10.5}}& \multicolumn{1}{c|}{6.7}  & \multicolumn{1}{c|}{7.3}  & 6.2 (XGB)           \\ \hline
9                                                    & \multicolumn{1}{c|}{\textbf{26.1}} & \multicolumn{1}{c|}{\textbf{\textcolor{Purple}{28.7}}} &\textcolor{NavyBlue}{33.7}& \multicolumn{1}{c|}{28.5}          & \multicolumn{1}{c|}{\textbf{\textcolor{Purple}{28.7}}} &\textcolor{NavyBlue}{\textbf{29.9}}& \multicolumn{1}{c|}{31.4} & \multicolumn{1}{c|}{29.1} & 28.6 (XGB)          \\ \hline
10                                                   & \multicolumn{1}{c|}{39.9}          & \multicolumn{1}{c|}{\textcolor{Purple}{43.9}}          &\textcolor{NavyBlue}{\textbf{46.3}}& \multicolumn{1}{c|}{\textbf{35.8}} & \multicolumn{1}{c|}{\textbf{\textcolor{Purple}{43.0}}} &\textcolor{NavyBlue}{47.8}& \multicolumn{1}{c|}{41.8} & \multicolumn{1}{c|}{34.8} & 33.4 (XGB)          \\ \hline
11                                                   & \multicolumn{1}{c|}{\textbf{39.1}} & \multicolumn{1}{c|}{\textbf{\textcolor{Purple}{39.3}}} &\textcolor{NavyBlue}{40.7}& \multicolumn{1}{c|}{\textbf{39.1}} & \multicolumn{1}{c|}{\textcolor{Purple}{39.4}}          &\textcolor{NavyBlue}{\textbf{40.6}}& \multicolumn{1}{c|}{39.5} & \multicolumn{1}{c|}{39.5} & 39.4 (XGB)          \\ \hline
12                                                   & \multicolumn{1}{c|}{23.0}          & \multicolumn{1}{c|}{\textcolor{Purple}{26.1}}          &\textcolor{NavyBlue}{28.7}& \multicolumn{1}{c|}{\textbf{22.0}} & \multicolumn{1}{c|}{\textbf{\textcolor{Purple}{23.8}}} &\textcolor{NavyBlue}{\textbf{27.2}}& \multicolumn{1}{c|}{25.5} & \multicolumn{1}{c|}{22.7} & 22.0 (XGB)          \\ \hline
13                                                   & \multicolumn{1}{c|}{28.7}          & \multicolumn{1}{c|}{\textcolor{Purple}{30.1}}          &\textcolor{NavyBlue}{33.5}& \multicolumn{1}{c|}{\textbf{28.2}} & \multicolumn{1}{c|}{\textbf{\textcolor{Purple}{29.2}}} &\textcolor{NavyBlue}{\textbf{31.9}}& \multicolumn{1}{c|}{28.9} & \multicolumn{1}{c|}{28.0} & 28.0 (RF)           \\ \hline
14                                                   & \multicolumn{1}{c|}{\textbf{19.5}} & \multicolumn{1}{c|}{\textbf{\textcolor{Purple}{28.9}}} &\textcolor{NavyBlue}{\textbf{37.6}}& \multicolumn{1}{c|}{22.2}          & \multicolumn{1}{c|}{\textcolor{Purple}{29.2}}          &\textcolor{NavyBlue}{39.8}& \multicolumn{1}{c|}{23.4} & \multicolumn{1}{c|}{20.5} & 20.5 (RF)           \\ \hline
15                                                   & \multicolumn{1}{c|}{12.2}          & \multicolumn{1}{c|}{\textcolor{Purple}{14.8}}          &\textcolor{NavyBlue}{20.8}& \multicolumn{1}{c|}{\textbf{11.2}} & \multicolumn{1}{c|}{\textbf{\textcolor{Purple}{13.8}}} &\textcolor{NavyBlue}{\textbf{18.4}}& \multicolumn{1}{c|}{12.9} & \multicolumn{1}{c|}{10.6} & 9.7 (XGB)           \\ \hline
16                                                   & \multicolumn{1}{c|}{\textbf{27.1}} & \multicolumn{1}{c|}{\textbf{\textcolor{Purple}{28.0}}} &\textcolor{NavyBlue}{31.3}& \multicolumn{1}{c|}{27.8}          & \multicolumn{1}{c|}{\textcolor{Purple}{28.6}}          &\textcolor{NavyBlue}{\textbf{31.0}}& \multicolumn{1}{c|}{27.8} & \multicolumn{1}{c|}{28.1} & 27.4 (ResNet)       \\ \hline
\end{tabular}
    \vspace{1cm}
    \caption{Numerical comparisons for the tabular data experiments, to accompany Figure~\ref{fig:tabular-full}. We report the median test error (\%) over 100 i.i.d. training runs of MLPs and random forests. The 3 subcolumns in each group denote models trained on the full dataset and their \{10\%, 1\%\} downsampled counterparts. For all of these results, bootstrap $95\%$ confidence intervals have width $< 2\%$. We observe that wide and/or sparsely-initialized MLPs are competitive with tree-based methods.
    In the rightmost 3 columns, we provide test errors for the same tasks, reported by \citep{grinsztajn2022why} (GBT = gradient boosted tree, XGB = XGBoost, FTT = Feature Tokenizer + Transformer \citep{gorishniy2021revisiting}). Note that our cross-validation protocol differs slightly (to handle variance incurred by downsampling), which may explain performance discrepancies. We include these only to illustrate that our full-data accuracies are commensurate with those in prior works focused exclusively on benchmarking models for tabular data.
    }
    \label{tab:tabular-results}
\end{table}

\paragraph{Results.} We note the following findings from the results in Figure~\ref{fig:tabular-full} and Table~\ref{tab:tabular-results}:

\begin{itemize}[leftmargin=3em]
    \item[(2T)] \textbf{Wide networks resist overfitting on small tabular datasets.} Like in the synthetic experiments, width yields monotonic end-to-end benefits for sample-efficient learning on these datasets. This suggests that the ``parallel search + pruning'' mechanisms analyzed in our paper are empirically at play in these settings, and that these networks' capacity to overfit does not preclude nontrivial feature learning and generalization. These comparisons can be seen by comparing the colored curves (which represent depths $2$ and $3$) within each subplot in Figure~\ref{fig:tabular-full}. In some (but not all) cases, these MLPs perform competitively with hyperparameter-tuned random forest classifiers.
    \item[(3T)] \textbf{Sparse axis-aligned initialization sometimes improves end-to-end performance.} These comparisons can be seen by comparing vertically adjacent sparse vs. dense subplots in Figure~\ref{fig:tabular-full}. This effect is especially pronounced on datasets which are downsampled to be orders of magnitude smaller. We believe that this class of drop-in replacements for standard initialization merits further investigation, and may contribute to closing the remaining performance gap between deep learning and tree ensembles on small tabular datasets.
\end{itemize}

\subsection{Software, compute infrastructure, and resource costs}
\label{subsec:infra}

GPU-accelerated training and evaluation pipelines were implemented in PyTorch~\citep{paszke2019pytorch}. Each training run was performed on one GPU in an internal cluster, with NVIDIA P40, P100, V100, and RTX A6000 GPUs. A single $T = 10^5$ training run took 10 seconds on average (with early termination for the vast majority of grid sweeps); across all of the results in this paper, around $2 \times 10^5$ training runs were performed, in a total of $600$ GPU-hours. Note that the precise evaluation of test error (at batch size $10^4$) constitutes a significant portion of the computational cost; this necessitated mindful GPU parallelization.

\end{document}